\documentclass{article}
\usepackage[hidelinks]{hyperref}
\usepackage{appendix}
\usepackage{pgfplots,booktabs,multirow}
\usepackage{siunitx}
\usepackage{array}
\newcolumntype{x}[1]{>{\centering\arraybackslash\hspace{0pt}}p{#1}}
\usepackage{graphicx,subcaption,wrapfig,floatflt,float}
\usepackage{xspace}
\usepackage{url}

\usepackage[colorinlistoftodos,textsize=tiny]{todonotes}
\usepackage[final]{microtype}
\usepackage[normalem]{ulem}
\usepackage{enumitem}
\usepackage{algorithm,algpseudocode,algorithmicx}

\usepackage{xfrac,xspace}
\usepackage{amsthm}
\usepackage{mathtools}

\usepackage{thmtools}
\usepackage{thm-restate}

\usepackage[english]{babel}
\usepackage[T1]{fontenc}

\usepackage{lmodern}	% font definition
\usepackage{amsmath}	% math fonts
\usepackage{amsthm}
\usepackage{amsfonts}

\usepackage{tikz}
\usetikzlibrary{matrix}
\usetikzlibrary{calc}
\usetikzlibrary{positioning}
\usetikzlibrary{decorations.pathmorphing} % noisy shapes
\usetikzlibrary{fit}					% fitting shapes to coordinates
\usetikzlibrary{backgrounds}	% drawing the background after the foreground
\usetikzlibrary{patterns}

\tikzset{
  matrixstyle/.style={
    matrix of nodes,
    nodes in empty cells,
    column sep      = -\pgflinewidth,
    row sep         = -\pgflinewidth,
    nodes={%
      inner sep=0mm,outer sep=0pt,
      minimum size=1.05cm,
      text height=\ht\strutbox,text depth=\dp\strutbox,
      draw,
      color=gray!40,
      line width=.5pt
    }
}}

\newcommand{\up}{8mm}
\newcommand{\xshift}{12mm}

\tikzstyle{background}=[rectangle,
                                                fill=gray!80,
                                                inner sep=0cm,
                                                ]
\tikzstyle{background2}=[rectangle,
                                                fill=gray!80,
                                                inner sep=0.cm,
                                                ]
                                                
\tikzstyle{state}=[circle,
      inner sep=0mm,outer sep=0pt,
      minimum size=1.05cm,
      text height=\ht\strutbox,text depth=\dp\strutbox,
      draw,
      color=gray!40,
      line width=.5pt
                                                ]
                                                
\tikzstyle{myarrow}=[circle,
                                                line width=1pt,
                                                inner sep=0cm,
                                                ]

\tikzstyle{background}=[rectangle,
                                                fill=gray!80,
                                                inner sep=0cm,
                                                ]
\tikzstyle{background2}=[rectangle,
                                                fill=gray!80,
                                                inner sep=0.cm,
                                                ]
                                                
\tikzstyle{myarrow}=[circle,
                                                line width=1pt,
                                                inner sep=0cm,
                                                ]

\tikzstyle{background}=[rectangle,
                                                fill=gray!80,
                                                inner sep=0cm,
                                                ]
\tikzstyle{background2}=[rectangle,
                                                fill=gray!80,
                                                inner sep=0.cm,
                                                ]
                                                
\tikzstyle{state}=[circle,
      inner sep=0mm,outer sep=0pt,
      minimum size=1.05cm,
      text height=\ht\strutbox,text depth=\dp\strutbox,
      draw,
      color=gray!100,
      line width=1pt
                                                ]
                                                
\tikzstyle{myarrow}=[circle,
                                                line width=1pt,
                                                inner sep=0cm,
                                                ]

\usepackage{amsmath,amssymb}
\usepackage{bbm}
\usepackage{bm}

\usepackage[normalem]{ulem}

\newcommand{\states}{\mathcal{S}}
\newcommand{\actions}{\mathcal{A}}
\newcommand{\T}{\mathcal{T}}
\newcommand{\E}{\mathbb{E}}

\newcommand{\bbR}{\mathbb{R}}
\newcommand{\calP}{\mathcal{P}}
\newcommand{\bbE}{\mathbb{E}}

\newcommand{\bbN}{\mathbb{N}}

\newcommand{\cov}[1]{\mathrm{Cov} ( #1 )}
\newcommand{\var}{\mathbb{V}}
\newcommand{\bbV}{\mathbb{V}}
\DeclareMathOperator*{\argmax}{argmax}

\newcommand{\aup}{\ensuremath{\textsc{up}}\xspace}
\newcommand{\adown}{\ensuremath{\textsc{down}}\xspace}

\newcommand{\qup}[1]{\hat{Q}(s_{#1}, \aup)}
\newcommand{\qdown}[1]{\hat{Q}(s_{#1}, \adown)}
\newcommand{\rup}[1]{\hat{r}(s_{#1}, \aup)}

\usepackage[final]{style/neurips_2019}

\usepackage[utf8]{inputenc} % allow utf-8 input
\usepackage[T1]{fontenc}    % use 8-bit T1 fonts
\usepackage{nicefrac}       % compact symbols for 1/2, etc.

\makeatletter
\newcommand{\specificthanks}[1]{\@fnsymbol{#1}}% Inserts a specific \thanks symbol
\makeatother

\newtheorem{definition}{Definition}
\newtheorem{proposition}{Proposition}

\declaretheorem[name=Proposition]{prop}
\declaretheorem[name=Lemma,numberlike=prop]{lemma}

\title{Successor Uncertainties: Exploration and Uncertainty in Temporal Difference Learning}

\author{%
    David Janz\thanks{Equal contribution}\textsuperscript{\hspace{0.55em}}\thanks{Work partly done during an internship at Microsoft Research Cambridge}\\
    University of Cambridge\\
    \texttt{dj343@cam.ac.uk}
    \And
    Jiri Hron\footnotemark[1]\\
    University of Cambridge\\
    \texttt{jh2084@cam.ac.uk}
    \And 
    Przemys\l{}aw Mazur\\
    Wayve Technologies
    \And
    Katja Hofmann\\
    Microsoft Research
    \And
    Jos{\'e} Miguel Hern{\'a}ndez-Lobato\\
    University of Cambridge\\
    Alan Turing Institute\\
    Microsoft Research
    \And
    Sebastian Tschiatschek\\
    Microsoft Research
}

\begin{document}

\maketitle
\begin{abstract}
Posterior sampling for reinforcement learning (PSRL) is an~effective method for balancing exploration and exploitation in reinforcement learning. Randomised value functions (RVF) can be viewed as a promising approach to scaling PSRL. However, we show that most contemporary algorithms combining RVF with neural network function approximation do not possess the~properties which make PSRL effective, and provably fail in sparse reward problems. Moreover, we find that propagation of uncertainty, a~property of PSRL previously thought important for exploration, does not preclude this failure. We use these insights to design Successor Uncertainties (SU), a cheap and easy to implement RVF algorithm that retains key properties of PSRL. SU is highly effective on hard tabular exploration benchmarks. Furthermore, on the Atari 2600 domain, it surpasses human performance on 38 of 49 games tested (achieving a median human normalised score of 2.09), and outperforms its closest RVF competitor, Bootstrapped DQN, on 36 of those.
\end{abstract}

\section{Introduction}
Perhaps the most important open question within reinforcement learning is how to effectively balance exploration of an unknown environment with exploitation of the~already accumulated knowledge \citep{kaelbling1996reinforcement,sutton1998reinforcement,busoniu2017reinforcement}.
In this paper, we study this in the classic setting where the~unknown environment is modelled as a~Markov Decision Process (MDP). 

Specifically, we focus on developing an algorithm that combines effective exploration with neural network function approximation.
Our approach is inspired by \textit{Posterior Sampling for Reinforcement Learning} \citep[PSRL;][]{strens2000abayesian,osband2013more}.
PSRL approaches the~exploration/exploitation trade-off by explicitly accounting for uncertainty about the~true underlying MDP.
In tabular settings, PSRL achieves impressive results and close to optimal regret \citep{osband2013more,osband2016lower}. 
However, many existing attempts to scale PSRL and combine it with neural network function approximation sacrifice the very aspects that make PSRL effective.
In this work, we examine several of these algorithms in the context of PSRL and:
\begin{enumerate}[itemsep=2pt,topsep=0pt]
    \item Prove that a previous avenue of research, propagation of uncertainty \citep{o2017uncertainty}, is neither sufficient nor necessary for effective exploration under posterior sampling. 
    \item Introduce \emph{Successor Uncertainties} (SU), a~cheap and scalable model-free exploration algorithm that retains crucial elements of the PSRL algorithm. 
    \item Show that SU is highly effective on hard tabular exploration problems.
    \item Present Atari 2600 results: SU outperforms Bootstrapped DQN \citep{osband2016deep} on $36/49$ and Uncertainty Bellman Equation \citep{o2017uncertainty} on $43/49$ games.
\end{enumerate}

\section{Background}\label{sec:background}

We use the following notation: for $X$ a random variable, we denote its distribution by $P_X$. Further, if $f$ is a~measurable function, then $f(X)$ follows the~distbution $f_\# P_X$ (the pushforward of $P_X$ by $f$).

We consider finite MDPs: a tuple $(\states, \actions, \T)$, where $\states$ is a finite state space, $\actions$ a finite action space, and $\T \colon \states \times \actions \to \mathcal{P}(\states \times \mathcal{R})$ a~transition probability kernel mapping from the~state-action space $\states \times \actions$ to the~set of probability distributions $\mathcal{P}(\states \times \mathcal{R})$ on the~product space of states $\states$ and rewards $\mathcal{R} \subset \bbR$; 
$\mathcal{R}$ is assumed to be bounded throughout.
For each time step $t \in \bbN$, the~agent selects an action $A_t$ by sampling from a~distribution specified by its policy $\pi \colon \states \to \calP(\actions)$ for the~current state $S_t$, and receives a~new state and reward $(S_{t+1}, R_{t+1}) \sim \mathcal{T} (S_t, A_t)$.
This gives rise to a Markov process $(S_t, A_t)_{t \geq 0}$ and a reward process $(R_t)_{t \geq 1}$. The~task of solving an MDP amounts to finding a~policy $\pi^\star$ which maximises the expected return $\smash{\E( \sum_{\tau=0}^\infty \gamma^\tau R_{\tau+1})}$ with $\gamma \in [0,1)$.

Crucial to many so called \emph{model-free methods} for solving MDPs is the~state-action value function (\emph{Q~function}) for a policy $\pi$:
$\smash{Q_t^{\pi} \coloneqq \E_t (\sum_{\tau=t}^\infty \gamma^{\tau - t} R_{\tau+1}) = \bbE_t (R_{t+1}) + \gamma \E_t ( Q_{t+1}^\pi} ) \, ,$ where $\bbE_t$ is used to denote an expectation conditional on $(S_\tau, A_\tau)_{\tau \leq t}$. Model-free methods use the~recursive nature of the~Bellman equation to construct a~model $\smash{\hat{Q}^\pi \colon \states \times \actions \to \bbR}$, which estimates $\smash{Q_t^\pi}$ for any given $\smash{(S_t=s, A_t=a)}$, through repeated application of the~\emph{Bellman operator} $T^\pi \colon \bbR^{\states \times \actions} \to \bbR^{\states \times \actions}$:
\begin{equation}\label{eq:bellman_operator}
    (T^\pi \hat{Q})(s, a) = \bbE_{(S', R') \sim \mathcal{T}(s, a)} [ R' + \gamma \bbE_{A'\sim \pi(S')} \hat{Q}(S', A')] \, .
\end{equation}
Since $\smash{T^\pi}$ is a~contraction on $\smash{\bbR^{\states \times \actions}}$ with a~unique fixed point $\smash{\hat{Q}^\pi}$, that is $\smash{T^\pi \hat{Q}^\pi = \hat{Q}^\pi}$, the~iterated application of $T^\pi$ to any initial $\smash{\hat{Q}}\in \bbR^{\states \times \actions}$ yields $\smash{\hat{Q}^\pi}$.
The expectations in equation~\eqref{eq:bellman_operator} can be estimated via Monte Carlo using experiences $(s, a, r, s^\prime)$ obtained through interaction with the MDP. 
A key challenge is then in obtaining experiences that are highly informative about the optimal policy.

A simple and effective approach to collecting such experiences is PSRL, a model-based algorithm based on two components: (i)~a~distribution over rewards and transition dynamics $P_{\hat{\mathcal{T}}}$ obtained using a~Bayesian modelling approach, treating rewards and transition probabilities as random variables; and (ii)~the~\emph{posterior sampling} exploration algorithm \citep{thompson1933likelihood,dearden1998bayesian} which samples $\smash{\hat{\mathcal{T}} \sim P_{\hat{\mathcal{T}}}}$, computes the~optimal policy $\hat{\pi}$ with respect to the~sampled $\smash{\hat{\mathcal{T}}}$, and follows $\hat{\pi}$ for the~duration of a~single episode. The~collected data are then~used to update the~$\smash{P_{\hat{\mathcal{T}}}}$ model, and the~whole process is iterated until convergence. 

While PSRL performs very well on tabular problems, it is computationally expensive and does not utilise any additional information about the~state space structure (e.g.\ visual similarity when states are represented by images).
A family of methods called Randomised Value Functions \citep[RVF;][]{osband2014generalization} attempt to overcome these issues by directly modelling a~distribution over Q~functions, $\smash{P_{\hat{Q}}}$, instead of over MDPs, $\smash{P_{\hat{\mathcal{T}}}}$.
Rather than acting greedily with respect to a~sampled MDP as in PSRL, the~agent then acts greedily with respect to a~sample $\smash{\hat{Q} \sim P_{\hat{Q}}}$ drawn at the~beginning of each episode, removing the~main computational bottleneck.
Since a~parametric model is often chosen for $\smash{P_{\hat{Q}}}$, the~switch to Q~function modelling also directly facilitates use of function approximation and thus generalisation between states.

\section{Exploration under function approximation}\label{sec:rpi}
Many exploration methods, including \citep{osband2014generalization,osband2016deep,moerland2017efficient,o2017uncertainty,azizzadenesheli2018efficient}, can be interpreted as combining the concept of RVF with neural network function approximation.
While the use of neural network function approximation allows these methods to scale to problems too complex for PSRL, it also brings about conceptual difficulties not present within PSRL and tabular RVF methods. Specifically, because a Q~function is defined with respect to a particular policy, constructing $\smash{P_{\hat{Q}}}$ requires selection of a reference policy or distribution over policies.
Methods that utilise a distribution over reference policies typically employ a bootstrapped estimator of the Q~function as we will discuss in more depth later. For now, we focus on methods that employ a single reference policy which commonly interleave two steps: (i)~inference of $\smash{P_{\hat{Q}^{\pi_i}}}$ for a~given policy $\pi_i$ using the available data (\emph{value prediction step}); (ii)~estimation of an improved policy $\pi_{i + 1}$ based on $\smash{P_{\hat{Q}^{\pi_i}}}$ (\emph{policy improvement step}).
While a common policy improvement choice is $\smash{\pi_{i+1} \colon s \mapsto \E_{P_{\hat{Q}^{\pi_i}}} [ G (\hat{Q})(s) ] }$, methods vary greatly in how they implement value prediction.
To gain a~better insight into the~value prediction step, we examine its idealised implementation:
Suppose we have access to a belief over MDPs, $\smash{P_{\hat{\mathcal{T}}}}$ (as in PSRL), and want to compute the~implied distribution $\smash{P_{\hat{Q}^\pi}}$ for a \emph{single} policy $\pi$.
The intuitive (albeit still computationally expensive) procedure is to: (i)~draw $\smash{\hat{\mathcal{T}} \sim P_{\hat{\mathcal{T}}}}$; and (ii)~repeatedly apply the~Bellman operator $T^\pi$ to an initial $\smash{\hat{Q}}$ for the~drawn $\smash{\hat{\mathcal{T}}}$ until convergence.
Denoting by $\smash{F^\pi\!\colon \hat{\mathcal{T}} \mapsto \hat{Q}^\pi}$ the map from $\smash{\hat{\mathcal{T}}}$ to the corresponding $\smash{\hat{Q}^\pi}$ for a policy $\pi$, the~distribution of resulting samples is $\smash{P_{\hat{Q}^\pi} = F_\#^\pi P_{\hat{\mathcal{T}}}}$. 

% \citep{Kearns2002,osband2017deep} deep exploration citations
This idealised value prediction step motivates, for example, the \emph{Uncertainty Bellman Equation} \citep[UBE;][]{o2017uncertainty}. \citeauthor{o2017uncertainty}\ argue that to achieve effective exploration, it is necessary that the~uncertainty about each $\smash{\hat{Q}^\pi (s, a)}$, quantified by variance, is equal to the~uncertainty about the~immediate reward and the~next state's Q~value. This requirement can be formalised as follows:
\begin{definition}[Propagation of uncertainty]\label{def:prop_uncert}
    For a~given distribution $\smash{P_{\hat{\mathcal{T}}}}$ and policy $\pi$, we say that a~model $\smash{P_{\hat{Q}^\pi}}$ \emph{propagates uncertainty} according to $\smash{P_{\hat{\mathcal{T}}}}$ if for each $(s, a) \in \states \times \actions$ and $p=1, 2$
    \begin{equation*}
        \mathbb{E}_{P_{\hat{Q}^\pi}} [\hat{Q}^\pi(s, a)^p]
        =
        \mathbb{E}_{F_\#^\pi P_{\hat{\mathcal{T}}}} [\hat{Q}^\pi(s, a)^p]
        =
        \mathbb{E}_{
            P_{\hat{\mathcal{T}}} 
        }
        \! \bigl\{
            [\bbE_{(R', S') \sim \hat{\mathcal{T}}(s,a)} R' + \bbE_{A' \sim \pi (S')} F^\pi (\hat{\mathcal{T}})(S', A') ]^p
        \bigr\} \!
        \, .
    \end{equation*}
\end{definition}
In words, \emph{propagation of uncertainty} requires that the first two moments behave consistently under application of the Bellman operator. 

Propagation of uncertainty is a desirable property when using \emph{Upper Confidence Bounds} \citep[UCB;][]{auer2002using} for exploration, since UCB methods rely only on the first two moments of $\smash{P_{\hat{Q}^\pi}}$. However, propagation of uncertainty is not sufficient for effective exploration under posterior sampling. We show this in the context of the~binary tree MDP depicted in figure~\ref{fig:tree-mdp}.
To solve the~MDP, the~agent must execute a~sequence of $L$ uninterrupted \aup movements.
In the following proposition, we show that any~algorithm combining factorised symmetric distributions with posterior sampling (e.g.\ UBE) will solve this MDP with probability of at most $2^{-L}$ per episode, thus failing to outperform a uniform exploration policy. Importantly, the~sizes of marginal variances have no bearing on this result, meaning that propagation of uncertainty on its own does not preclude this failure mode.

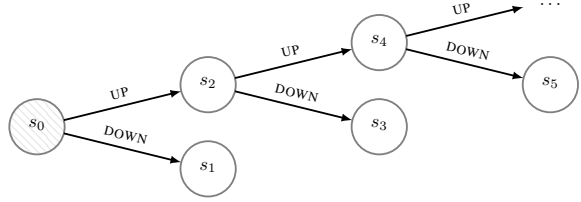
\begin{figure}[tb]        
    \centering\vspace{-2mm}
    \scalebox{0.7}{\begin{tikzpicture}[>=latex,text height=0.5ex]
   % chain
   \node[state, pattern=north west lines, pattern color=gray!20] (s0) {};
   \node[state,right=of s0, xshift=\xshift, yshift=\up] (s1r) {};
   \node[state,right=of s0, xshift=\xshift, yshift=-\up] (s1l) {};
   \node[state,right=of s1r, xshift=\xshift, yshift=\up] (s2r) {};
   \node[state,right=of s1r, xshift=\xshift, yshift=-\up] (s2l) {};
   \node[state,draw=none,right=of s2r, xshift=\xshift, yshift=\up, color=black] (s3r) {\reflectbox{$\ldots$}};
   % \node[state,draw=none,right=of s2r, xshift=\xshift, yshift=-\up, color=black] (s3l) {$\cdots$};
   \node[state,right=of s2r, xshift=\xshift, yshift=-\up] (s3l) {};
   
   % \node[yshift=8mm] at (s0) {$s_0$};
   \node at (s0) {$s_0$};
   \node at (s1r) {$s_2$};
   \node at (s1l) {$s_1$};
   \node at (s2r) {$s_4$};
   \node at (s2l) {$s_3$};
   \node at (s3l) {$s_5$};
   
   % arrows
   \path[->,myarrow] (s0)  edge node[pos=0.5, sloped, yshift=2mm] (ls1l) {\footnotesize \textsc{up}} (s1r);
   \path[->,myarrow] (s0)  edge node[pos=0.5, sloped, yshift=2mm] (ls1l) {\footnotesize \textsc{down}} (s1l);
   
   \path[->,myarrow] (s1r)  edge node[pos=0.5, sloped, yshift=2mm] (ls1l) {\footnotesize \textsc{up}} (s2r);
   \path[->,myarrow] (s1r)  edge node[pos=0.5, sloped, yshift=2mm] (ls1l) {\footnotesize \textsc{down}} (s2l);
   
   \path[->,myarrow] (s2r)  edge node[pos=0.5, sloped, yshift=2mm] (ls1l) {\footnotesize \textsc{up}} (s3r);
   \path[->,myarrow] (s2r)  edge node[pos=0.5, sloped, yshift=2mm] (ls1l) {\footnotesize \textsc{down}} (s3l);

\end{tikzpicture}}
    \caption{Binary tree MDP of size $L$. States $\states = \{s_0, \dots, s_{2L} \}$ are one-hot encoded; actions $\actions = \{a_1, a_2\}$ are mapped to movements $\{\aup, \adown \}$ according to a~random mapping drawn independently for each state.
    Reward of one is obtained after reaching $s_{2L}$ and zero otherwise. States with odd indices and $s_{2L}$ are terminal.}
    \label{fig:tree-mdp}
\end{figure}

\begin{restatable}{prop}{factorised}\label{prop:factorised_symm}
   Let $|\actions| > 1$, and $\smash{P_{\hat{Q}}}$ be a~factorised distribution, i.e.\ for $\smash{\hat{Q} \sim P_{\hat{Q}}}$, $\smash{\hat{Q}(s, a)}$ and $\smash{\hat{Q}(s', a')}$ are independent, $\forall (s, a) \neq (s', a')$, with symmetric marginals.
   Assume that for each $s \in \states$, the~marginal distributions of $\smash{\{ \hat{Q}(s, a) \colon a \in \actions \}}$ are all symmetric around the~same value $c_s \in \mathbb{R}$.
   Then the~probability of executing any given sequence of $L$ actions under $\hat{\pi} \sim \smash{G_\# P_{\hat{Q}}}$ is at most $\smash{2^{-L}}$.
\end{restatable}

Propagation of uncertainty is furthermore not necessary for posterior sampling.
To see this, first note that for any given $\smash{P_{\hat{Q}^\pi}}$, the~posterior sampling procedure only depends on the~induced distribution over greedy policies, i.e. the~pushforward of $\smash{P_{\hat{Q}^\pi}}$ by the greedy operator $G$.
This means that from the~point of view of posterior sampling, two Q~function models are equivalent as long as they induce the~same distribution over greedy policies.
In what follows, we formalise this equivalence relationship (definition~\ref{def:incorp_dep}), and then show that each of the~induced equivalence classes contains a~model that does not propagate uncertainty (proposition~\ref{prop:prop_uncert_unnecessary}), implying that posterior sampling does \emph{not} rely on propagation of uncertainty.

\begin{definition}[Posterior sampling policy matching]\label{def:incorp_dep}
    For a~given distribution $\smash{P_{\hat{\mathcal{T}}}}$ and a~policy $\pi$, we say that a~model $\smash{P_{\hat{Q}^\pi}}$ matches the posterior sampling policy implied by $\smash{P_{\hat{\mathcal{T}}}}$ if $\smash{G_\# P_{\hat{Q}^\pi} = (G \circ F^\pi)_\# P_{\hat{\mathcal{T}}}}$.
\end{definition}

\begin{restatable}{prop}{propunnecessary}\label{prop:prop_uncert_unnecessary}
    For any distribution $\smash{P_{\hat{\mathcal{T}}}}$ and policy $\pi$ such that the~variance $\smash{\var_{F_\#^\pi P_{\hat{\mathcal{T}}}} [\hat{Q}^\pi (s, a)]}$ is greater than zero for some $(s, a)$, there exists a distribution $\smash{P_{\hat{Q}^\pi}}$ which \emph{matches the~posterior sampling policy} (definition~\ref{def:incorp_dep}), but does not \emph{propagate uncertainty} (definition~\ref{def:prop_uncert}), according to $\smash{P_{\hat{\mathcal{T}}}}$.
\end{restatable}

We conclude by addressing a~potential criticism of proposition~\ref{prop:factorised_symm}, i.e.\ that the~described issues may be circumvented by initialising expected Q~values to a~value higher than the~maximal attainable Q~value in given MDP, an~approach known as optimistic initialisation \citep{osband2014generalization}.
In such case, symmetries in the Q~function may break as updates move the~distribution towards more realistic Q~values.
However, when neural network function approximation is used, the~effect of optimistic initialisation can disappear quickly with optimisation~\citep{osband2018randomized}.
In particular, with non-orthogonal state-action embeddings, Q~value estimates may decrease for yet unseen state-action pairs, and estimates for different state-action states can move in tandem.
In practice, most recent models employing neural network function approximation do not use optimistic initialisation \citep{osband2016deep,azizzadenesheli2018efficient,moerland2017efficient,o2017uncertainty}.

\section{Successor Uncertainties}\label{sec:implementation}
We present \emph{Successsor Uncertainties}, an~algorithm which both propagates uncertainty and matches the~posterior sampling policy.
As our work is motivated by PSRL, we focus on the~use with posterior sampling, leaving combination with other exploration algorithms for future research.

\subsection{Q~function model definition} 

Suppose we are given an~embedding function $\phi \colon \states \times \actions \to \bbR^d$, such that for all $(s,a)$, $\|\phi(s, a)\|_2 = 1$ and $\phi(s, a) \geq 0$ elementwise, and $\bbE_t R_{t+1} = \langle \phi_{t} ,w \rangle$ for some $w \in \bbR^d$. Denote $\phi_t = \phi(S_t, A_t)$. Then we can express $Q^\pi_t$ as an inner product of $w$ and $\psi^\pi_t = \bbE_t [\sum_{\tau=t}^{\infty} \gamma^{\tau -t} \phi_{\tau}]$, the (discounted) expected future occurrence of each $\smash{\phi(s,a)}$ feature under a policy $\pi$, as follows:
\begin{align}\label{eq:sf-deriv}
    Q^\pi_t 
    &= \E_t \sum_{\tau=t}^{\infty} \gamma^{\tau-t} R_{\tau+1} 
    = \E_t \sum_{\tau=t}^{\infty} \gamma^{\tau - t} \langle \phi_{\tau}, w \rangle
    = \biggl\langle \E_t \sum_{\tau=t}^{\infty} \gamma^{\tau-t} \phi_{\tau} , w \biggr\rangle 
    = \langle \psi^\pi_t, w \rangle 
    \,,
\end{align}
where the~second equality follows from the tower property of conditional expectation and the~third from the~dominated convergence theorem combined with the~unit norm assumption. 

The~quantity $\psi^\pi_t$ is known in the literature as the \emph{successor features}~\citep{dayan1993improving,barreto2017successor}. Noting that $\smash{\psi^\pi_t = \phi_{t} + \gamma \, \bbE_t \psi^\pi_{t+1}}$, an estimator of the successor features, $\smash{\hat{\psi}^\pi}$, can be obtained by applying standard temporal difference learning techniques. The other quantity involved, $w$, can be estimated by regressing embeddings of observed states $\phi_t$ onto the corresponding rewards. We perform Bayesian linear regression to infer a distribution over rewards, using $\mathcal{N} (0, \theta I)$ as the~prior over $w$ and $\mathcal{N}(\langle \phi, w \rangle, \beta)$ as the~likelihood, which leads to posterior $\smash{\mathcal{N}(\mu_w, \Sigma_w)}$ over $w$ with known analytical expressions for both $\mu_w$ and $\Sigma_w$. This~induces posterior distribution over $\smash{\hat{Q}_\textrm{SU}^\pi}$ given by 
\begin{equation}\label{eq:su_prob_dist}
    \hat{Q}_{\textnormal{SU}}^\pi \sim \mathcal{N} ( \hat{\Psi}^\pi \mu_w , \hat{\Psi}^\pi \Sigma_w (\hat{\Psi}^\pi)^\top)
    \, ,
\end{equation}
where $\smash{\hat{\Psi}^\pi = [\hat{\psi}^\pi(s,a)]_{(s,a) \in \states \times \actions}^\top}$. This is our Successor Uncertainties (SU) model for the Q~function.

The final element of the SU model is the selection of a sequence of reference policies $(\pi_i)_{i \geq 1}$ for which the Q~function model is learnt. We follow \citet{o2017uncertainty} in constructing these iteratively as $\smash{\pi_{i + 1} (s) = \E_{\hat{\pi} \sim G_\# P_{\hat{Q}^{\pi_i}}} [\hat{\pi}(s)]}$.

\subsection{Properties of the model}

The non-diagonal covariance matrix of the SU Q~function model (see equation~\eqref{eq:su_prob_dist}) means that SU does not suffer from the shortcomings of previous methods with factorised posterior distributions described in proposition~\ref{prop:factorised_symm}. Moreover, note that $\smash{\hat{Q}_{\textnormal{SU}}^\pi \sim F_{\#}^\pi P_{\hat{\mathcal{T}}}}$ for the~MDP model $P_{\hat{\mathcal{T}}}$ composed of a~delta distribution concentrated on empirical transition frequencies, and the~Bayesian linear model for rewards (assuming convergence of successor features, i.e.\ $\smash{\hat{\psi}^\pi = \psi^\pi }$). SU thus both propagates uncertainty and matches the~posterior sampling policy according to this choice of $P_{\hat{\mathcal{T}}}$. 

However, due to its use of a point estimate for the~transition probabilities, SU may underestimate Q~function uncertainty, and a~good model of transition probabilities which scales beyond tabular settings can lead to improved performance.
Furthermore, SU estimates  $\smash{P_{\hat{Q}^{\pi_{i + 1}}}}$ for a single policy, which we choose to be $\smash{\pi_{i + 1} (s) = \E_{\hat{\pi} \sim G_\# P_{\hat{Q}^{\pi_i}}} [\hat{\pi}(s)]}$. This approach may not adequately capture the uncertainty over $\hat{\pi}$ implied by $\smash{P_{\hat{Q}^{\pi_i}}}$. We expect that incorporation of this uncertainty, or an~improved method of choosing $\pi_{i + 1}$, may further improve the~SU algorithm.

\subsection{Neural network function approximation}\label{sec:nn_approximation}

One of the~main assumptions we made so far is that the~embedding function $\phi$ is known a~priori.
This section considers the~scenario where $\phi$ is to be estimated jointly with the~other quantities using neural network function approximation.
For reference, the~pseudocode is included in appendix~\ref{a:implementation}.

Let $\smash{\hat{\phi} \colon \mathcal{S} \times \mathcal{A} \to \bbR_+^d}$ be the current estimate of $\phi$, $(s_t, a_t)$ the~state-action pair observed at step $t$, $r_{t+1}$ the reward observed after taking action $a_t$ in state $s_t$.
Suppose we want to estimate the~Q~function of some given policy $\pi$, and denote $\smash{\hat{\phi}_t \coloneqq \hat{\phi}(s_t, a_t)}$, $\smash{\hat{\psi}_t \coloneqq \hat{\psi}^\pi (s_t, a_t)}$.
We propose to jointly learn $\smash{\hat{\phi}}$ and $\smash{\hat{\psi}}$ by enforcing the~known relationships between $\phi_t$, $\psi_t^\pi$ and $\bbE_t R_{t+1}$:
\begin{align}\label{eq:objective}
    &\textrm{min}_{
        \hat{\phi}, \hat{\psi}, \hat{w}
    }
    \,
        \underbrace{\| \hat{\psi}_t - \hat{\phi}_{t} - \gamma \, (\hat{\psi}_{t+1})^\dagger \|_2^2}_{\text{successor feature loss}}
        + 
        \underbrace{|\langle \hat{w}, \hat{\phi}_t \rangle \mkern-2mu - \mkern-2mu r_{t+1}|^2}_{\text{reward loss}}
        + 
        \underbrace{|\langle \hat{w}, \hat{\psi}_t \rangle \mkern-2mu - \mkern-2mu \gamma (\langle \hat{w}, \hat{\psi}_{t+1} \rangle)^\dag \mkern-3mu - r_{t+1}|^2 }_{\text{Q~value loss}}
\end{align}
in expectation over the~observed data $\smash{\{(s_t, a_t, r_{t+1} s_{t+1}) \colon t=0,\ldots,N\}}$ with $a_{t+1} \sim \pi(s_{t+1})$;
$\smash{\hat{\phi}_t, \hat{\psi}_t \in \bbR_+^d}, \smash{\|\hat{\phi}_t\|_2 = 1, \forall t, }$ are respectively ensured by the use of ReLU activations and explicit normalisation.
The~$\smash{\hat{w} \in \bbR^d}$ are the~final layer weights shared by the the~reward and the~Q~value networks. Quantities superscripted with~$\smash{\dagger}$ are treated as fixed during optimisation.

The need for the successor feature and reward losses follows directly from the definition of the~SU model.
We add the explicit Q~value loss to ensure accuracy of Q~value predictions.
Assuming that there exists a~(ReLU) network that achieves zero successor feature and reward loss, the~added Q~value loss has no effect.
However, finding such an~optimal solution is difficult in practice and empirically the~addition of the~Q~value loss improves performance.
% However, in practice, convergence to such an~optimal solution is unlikely.
Our modelling assumptions cause all constituent losses in equation~\eqref{eq:objective} to have similar scale, and thus we found it unnecessary to introduce weighting factors.
Furthermore, unlike in previous work utilising successor features \citep{kulkarni2016deep,machado2017eigenoption,machado2018count}, SU does not rely on any auxiliary state reconstruction or state-transition prediction tasks for learning, which simplifies implementation and greatly reduces the~required amount of computation.

We employ the neural network output weights $\hat{w}$ in prediction of the mean $Q$ function, and use the~Bayesian linear model only to provide uncertainty estimates.
In estimating the covariance matrix $\Sigma_w$, we decay the contribution of old data-points, $\smash{\hat{\Sigma}_w = (\zeta^{N}\theta^{-1}I + \beta^{-1} \sum_{i=0}^N \zeta^{N-i} \hat{\phi}_i \hat{\phi}_i^\top)^{-1}} \, , \zeta \in [0, 1]$, so as to counter non-stationarity of the learnt state-action embeddings $\smash{\hat{\phi}}\,$.

\subsection{Comparison to existing methods}\label{sec:existing}

We discuss two popular classes of Q~function models compatible with neural network function approximation: methods relying on Bayesian linear Q~function models and methods based on bootstrapping. We omit variational Q-learning methods such as \citep{gal2016uncertainty,lipton2016efficient}, as conceptual issues with these algorithms have already been identified in an~illuminating line of work by~\citet{osband2016deep,osband2018randomized}.

Bayesian linear Q~function models encompass our SU algorithm, UBE~\citep{o2017uncertainty} implemented with value function approximation, Bayesian Deep Q~Networks~\citep[BDQN;][]{azizzadenesheli2018efficient}, and a range of other related work \citep{levine2017shallow,moerland2017efficient}.
The algorithms within this category tend to use a~Q~function model of the form $\smash{\hat{Q}^\pi(s, a) = \langle \hat{\phi}^\pi_s, w_a \rangle}$, where $\smash{\hat{\phi}^\pi_s}$ are state embeddings and $\smash{w_a \sim P_{w_a}}$ are weights of a~Bayesian linear model.
The embeddings $\smash{\hat{\phi}^\pi_s}$ are produced by a~neural network, and are usually optimised using a~temporal difference algorithm applied to Q~values.
However, these methods do not enforce any explicit structure within the~embeddings $\smash{\hat{\phi}^\pi_s}$ which would be required for posterior sampling policy matching, and prevent these methods from falling victim to proposition~\ref{prop:factorised_symm}. SU can thus be viewed as a~simple and computationally cheap alternative fixing the~issues of existing Bayesian linear Q~function models. 

Bootstrapped DQN \citep{osband2016deep,osband2018randomized} is a~model which consists of an ensemble of $K$ standard Q~networks, each initialised independently and trained on a~random subset of the~observed data.
Each network is augmented with a fixed additive prior network, so as to ensure the ensemble distribution does not collapse in sparse environments. If all networks within the ensemble are trained to estimate the Q~function for a single policy $\pi$, then Bootstrapped DQN both propagates uncertainty and matches the~posterior sampling policy for a distribution over MDPs formed by the~mixture over empirical MDPs corresponding to each subsample of the~data.
In practice, Bootstrapped DQN does not assume a single policy $\pi$ and instead each network learns for its corresponding greedy policy. Bootstrapped DQN is, however, more computationally expensive: its performance increases with the size of the ensemble $K$, but so does the amount of computation required. Our experiments show that SU is much cheaper computationally, and that despite using only a single reference policy, it manages to outperform Bootstrapped DQN on a~wide range of exploration tasks (see section~\ref{sec:experiments}).

\section{Tabular experiments}
\label{sec:experiments}

We present results for: (i)~the~binary tree MDP accompanied by theoretical analysis showing how SU succeeds and avoids the~pitfalls identified in proposition~\ref{prop:factorised_symm}; (ii)~a~hard exploration task proposed by \citet{osband2018randomized} together with the~Boostrapped DQN algorithm which SU outperforms by a~significant margin.\footnote{Code for the~tabular experiments: \url{https://djanz.org/successor_uncertainties/tabular_code}}
We also provide an analysis explaining why some of the~previously discussed algorithms perform well on seemingly similar experiments present in existing literature.

\subsection{Binary tree MDP}\label{sec:tree-experiments}

We study the behaviour of SU and its competitors on the~binary tree MDP introduced in figure~\ref{fig:tree-mdp}. Figure~\ref{fig:scaling} shows the~empirical performance of each algorithm as a~function of the tree size~$L$. 
Evidently, both BDQN and UBE fail to outperform a~uniform exploration policy.
For UBE, this is a consequence of proposition~\ref{prop:factorised_symm}, and the~similarly poor behaviour of BDQN suggests it may suffer from an~analogous issue.
In contrast, SU and Bootstrapped DQN are able to succeed on large binary trees despite the very sparse reward structure and randomised action effects. 
However, Bootstrapped DQN requires approximately 25 times more computation than SU to approach similar levels of performance due to the~necessity to train a~whole ensemble of Q~networks.

\begin{figure}[h]
\includegraphics[width=\textwidth,clip,trim={7 8 7 7}]{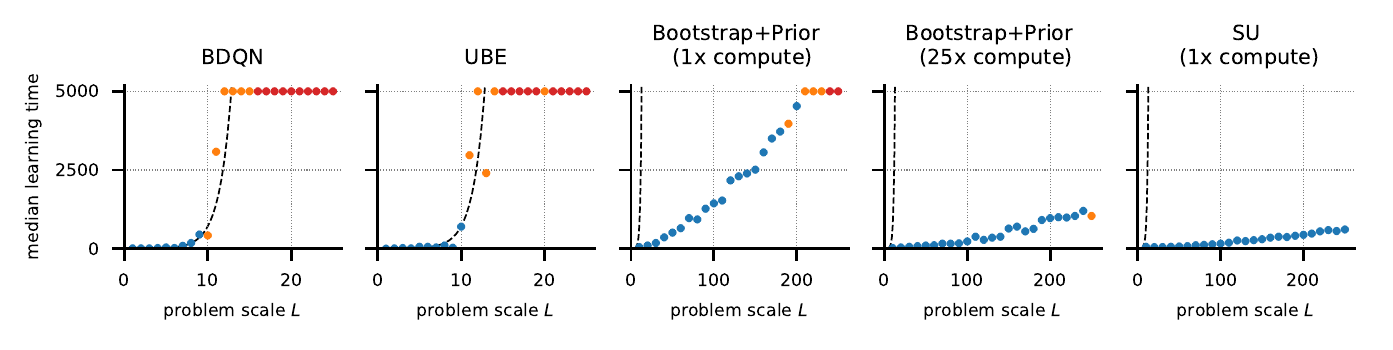}
\caption{Median number of episodes required to learn the optimal policy on the~tree MDP. Blue points indicate all 5 seeds succeeded within 5000 episodes, orange indicates only some of the~runs succeeded, and red all runs failed. Dashed lines correspond to the~median for a uniform exploration policy. Note the~reduced size of the x-axis for BDQN and UBE.}
\label{fig:scaling}
\end{figure}

The next proposition and its proof provide intuition for the success of SU on the~tree MDP.
The~proof is based on a~lemma stated just after the~proposition (see appendix~\ref{a:tree-proofs} for formal treatment).
\newcounter{suboundcounter}\setcounter{suboundcounter}{\value{prop}}
\setcounter{proposition}{\value{suboundcounter}}
\addtocounter{prop}{1}
\begin{proposition}[Informal statement]\label{prop:su-bound}
Assume the~SU model with: (i)~fixed one-hot state-action embeddings $\phi$, (ii)~uniform exploration thus far, (iii)~successor representations learnt to convergence for a uniform policy.
Let $s_{k}$ for $2 \leq k < 2L$, even, be a state visited $N$ times thus far.
Then the~probability of selecting $\aup$ in $s_k$, given $\aup$ was selected in $s_0 , s_2, \ldots , s_{k - 2}$, is greater than one half with probability greater than $1 - \epsilon_N$, where $\epsilon_N$ decreases exponentially with $N$. 
\end{proposition}

\begin{lemma}[Informal statement]\label{lemma:covariance_condition}
Under the~SU model $\smash{\hat{Q} \sim P_{\hat{Q}^\pi}}$ for the~uniform policy $\pi$, the~probability that the~greedy policy $\hat{\pi} = \smash{G(\hat{Q})}$ selects $\aup$ in $s_k$, given $\aup$ was selected in $s_0 , s_2, \ldots , s_{k - 2}$, is greater than one half if there exists an~even $0 \leq j < k$ such that 
\begin{equation*}
    \cov{\qup{k}, \qup{j}} > \cov{\qdown{k}, \qup{j}}
    \, .
\end{equation*} 
\end{lemma}
\begin{proof}[Sketch proof of proposition~\ref{prop:su-bound}]
Under SU $\smash{\qup{j}} = \allowdisplaybreaks \hat{r}(s_j, \aup) + \allowdisplaybreaks \ldots + \allowdisplaybreaks \rho \qup{k} + \allowdisplaybreaks \smash{\rho \qdown{k}}$ with $\smash{\rho=2^{-(\frac{k-j}{2})}}$ the probability of getting from $s_j$ to $s_k$ under the uniform policy. 
Note that $\smash{\qup{j}}$ and $\smash{\qdown{k}}$ only share the~$\smash{\qdown{k} = \hat{r}(s_k, \adown)}$ term, whereas $\smash{\qup{k}}$ and $\smash{\qup{j}}$ share $\smash{\hat{r}(s_j, \aup),\ldots,\hat{r}(s_{p}, \adown)}$, where $s_p$ is the state with the~highest index seen so far.
Thus covariance between $\smash{\qup{k}}$ and $\smash{\qup{j}}$ is higher than that between $\smash{\qdown{k}}$ and $\smash{\qup{j}}$ with high probability (at least $1-\epsilon_N$), and the result follows from lemma~\ref{lemma:covariance_condition}.
\end{proof}
 
Proposition~\ref{prop:su-bound} implies that (at least under the simplifying assumption of prior exploration being uniform) SU is likely to assign higher probability to Q~functions for which a greedy policy leads towards the~furthest visited state (cf.\ the~role of the~state $s_p$ in the~sketch proof). This is a strategy actively aimed for in exploration algorithms such as Go-Explore where the agent uses imitation learning to return to the~furthest discovered states \citep{ecoffet2018montezuma}. 

\subsection{Chain MDP from \citep{osband2018randomized}}\label{sec:chain-experiments}

We present results on the chain environment introduced by \citet{osband2018randomized}, described in detail in appendix~\ref{a:chain-experiments}. \citeauthor{osband2018randomized} describe their MDP as being ``akin to looking for a piece of hay in a needle-stack'' and state that it ``may seem like an impossible task''. Figure~\ref{fig:scaling-log-log} shows the scaling for Successor Uncertainties and Bootstrap+Prior for this problem.
Learning time $T$ scales empirically as $\smash{\mathcal{O}(L^{2.5})}$ for SU, versus $\smash{\mathcal{O}(L^3)}$ for Bootstrap+Prior \citep[as reported in][]{osband2018randomized}.

\begin{figure}[H]
    \centering
    \includegraphics[width=\textwidth,clip,trim={5 5 5 5}]{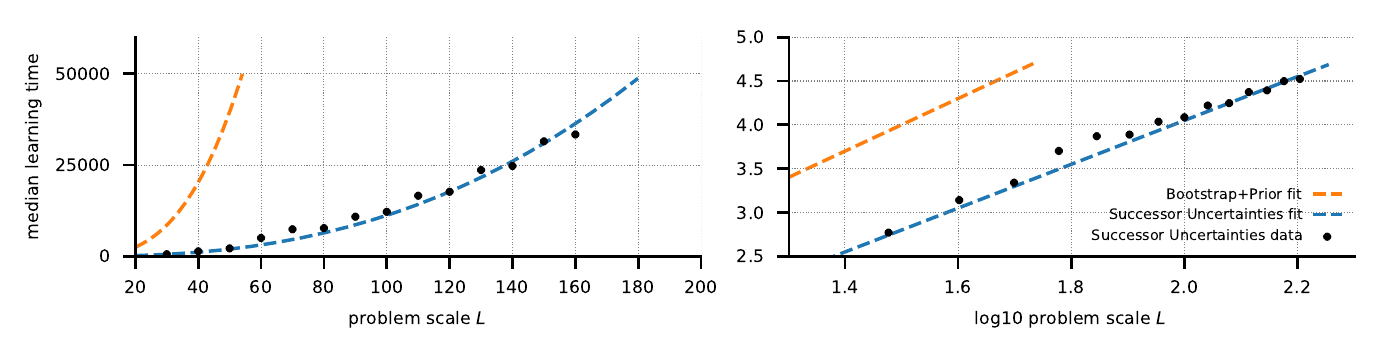}
    \caption{Learning time $T$ for SU and Bootstrap+Prior for a range of problem sizes $L$ on the chain MDP. Curve for SU is $\log_{10}T = 2.5\log_{10}L - 0.95$. Curve for Bootstrap+Prior is taken from figure~8 in \citep{osband2018randomized}. }
    \label{fig:scaling-log-log}
\end{figure}

\subsection{On the success of BDQN in environments with tied actions}\label{sec:tied-actions}
We briefly address prior results in the literature where BDQN is seen solving problems seemingly similar to our binary tree MDP with ease \citep[as in, for example, figure~1 of][]{touati2018randomized}. The~discrepancy occurs because previous work often does not randomise the effects of actions \citep[for~example][]{osband2016deep,plappert2017parameter,touati2018randomized}, i.e.\ if $a_1$ leads \aup in any state $s_k$, then $a_1$ leads \aup in all states. We refer to this as the~\emph{tied actions} setting.
In the following proposition, we show that MDPs with tied actions are trivial for BDQN with strictly positive activations (e.g.\ sigmoid).
We offer a similar result for ReLU in appendix~\ref{a:tied-action-proofs}.

\begin{restatable}{prop}{bdqnWinsTwo}\label{prop:bdqn-wins2}
Let $\smash{\hat{Q}(s, a) = \langle \phi(s), w_a \rangle}$ be a~Bayesian Q~function model with $\phi(s) = \varphi (U 1_s) \in \mathbb{R}^d$, $1_s$ a~one-hot encoding of $s$, and $\varphi$ a~strictly positive activation function (e.g.\ sigmoid) applied elementwise.
Then sampling independently from the~prior $w_{a} \sim \mathcal{N} (0 , \sigma_w^2 I )$, $U_{hs} \sim \mathcal{N} (0, \sigma_u^2)$ solves a~tied action binary tree of size $L$ in $\smash{T \leq - [\log_2(1-2^{-d})]^{-1}}$ median number of episodes.
\end{restatable}
\begin{proof}
Define $\Delta \coloneqq w_{\aup} - w_{\adown}$ and observe $\aup$ is selected if $\smash{\hat{Q}(s, \aup) - \hat{Q}(s , \adown)} = \langle \phi(s) , w_{\aup} - w_{\adown} \rangle > 0$.
By strict positivity of $\varphi$, the~probability that \aup is always selected
\begin{equation*}
    \mathbb{P}\bigl[\bigcap_{j=0}^{L-1} \{ \smash{\hat{Q}(s_{2j}, \aup)} \! > \! \smash{ \hat{Q}(s_{2j}, \adown)} \} \bigr] \!\geq\! \mathbb{P}\bigl[ \bigcap_{j=0}^{L-1} \{ \langle \phi(s_{2j}), \Delta \rangle \! > \! 0 \} \mid \Delta \! > \! 0\bigr] \mathbb{P}(\Delta \! > \! 0) = \mathbb{P}(\Delta > 0) \, ,
\end{equation*}
where $\Delta > 0$ is to be interpreted elementwise. As $\Delta \sim \mathcal{N}(0, 2\sigma_w^2 I)$, $\mathbb{P}(\Delta > 0) = 2^{-d}$ for all $L$.
\end{proof}

A~single layer BDQN with one neuron can thus solve a~tied action binary tree of any size $L$ in one episode (median) while completely ignoring all state information.
That such an~approach can be successful implies tied actions MDPs generally do not make for good exploration benchmarks.

\section{Atari 2600 experiments}\label{sec:atari-experiments}
We have tested the SU algorithm on the~standard set of 49 games from the Arcade Learning Environment, with the aim of showing that SU can be scaled to complex domains that require generalisation between states. 
We use a standard network architecture as in \citep{mnih2015human,van2016deep} endowed with an extra head for prediction of $\smash{\hat{\phi}}$ and one-step value updates.
More detail on our implementation, network architecture and training procedure can be found in appendix~\ref{a:atari-experiments}.\footnote{Code for the~Atari experiments: \url{https://djanz.org/successor_uncertainties/atari_code}}

SU obtains a median human normalised score of 2.09 (averaged over 3 seeds) after 200M training frames under the `no-ops start 30 minute emulator time' test protocol described in \citep{hessel2018rainbow}. Table~\ref{tab:atari-statistics} shows we significantly outperform competing methods.
The raw scores are reported in table~\ref{tab:atari-raw-scores} (appendix), and the difference in human normalised score between SU and the~competing algorithms for individual games is charted in figure~\ref{fig:atari-relative}.
Since \citet{azizzadenesheli2018efficient} only report scores for a~small subset of the~games and use a~non-standard testing procedure, we do not compare against BDQN. \citet{osband2018randomized}, who introduce Bootstrap+Prior, do not report Atari results; we thus compare with results for the~original plain Bootstrapped DQN \citep{osband2016deep} instead.

\begin{table}[h]
    \centering
    \caption{\label{tab:atari-statistics} Human normalised Atari scores. Superhuman performance is the~percentage of games on which each~algorithm surpasses human performance \citep[as reported in][]{mnih2015human}.}
    \begin{tabular}{lx{4.5em}x{4.5em}x{4.5em}x{6em}}
            \toprule
            \multirow{2}{*}{Algorithm}
            & \multicolumn{3}{c}{Human normalised score percentiles}
            & Superhuman \\
            & 25\% & 50\% & 75\% & performance \% \\
            \midrule
            Successor Uncertainties & \textbf{1.06} & \textbf{2.09} & \textbf{5.95} & \textbf{77.55\%} \\
            Bootstrapped DQN & 0.76 & 1.60 & 5.16 & 67.35\%  \\
            UBE & 0.38 & 1.07 & 4.14 & 51.02\% \\
            DQN + $\epsilon$-greedy & 0.50 & 1.00 & 3.41 & 48.98\% \\
            \bottomrule
    \end{tabular}
\end{table}

\begin{figure}[t]
    \centering
    \includegraphics[width=1\textwidth,clip,trim={3 4 0 5}]{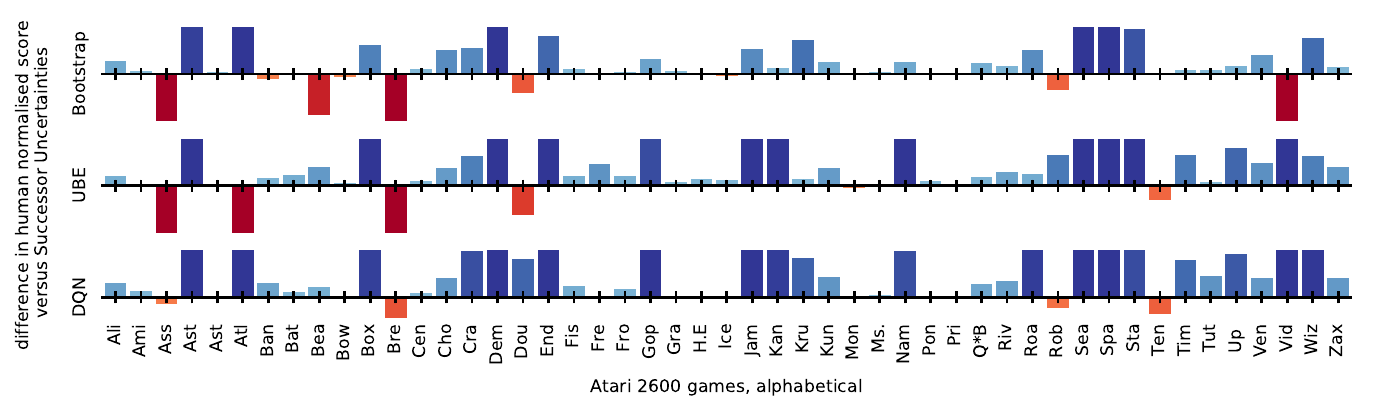}
    \caption{Bars show the difference in human normalised score between SU and Bootstrap DQN (top), UBE (middle) and DQN (bottom) for each of the 49 Atari 2600 games. Blue indicates SU performed better, red worse. SU outperforms the baselines on $36/49$, $43/49$ and $42/49$ games respectively. Y-axis values have been clipped to $[-2.5, 2.5]$.}
    \label{fig:atari-relative}
\end{figure}

\section{Conclusion}

We studied the Posterior Sampling for Reinforcement Learning algorithm and its extensions within the Randomised Value Function framework, focusing on use with neural network function approximation.
We have shown theoretically that exploration techniques based on the concept of propagation of uncertainty are neither sufficient nor necessary for posterior sampling exploration in sparse environments. 
We instead proposed posterior sampling policy matching, a~property motivated by the~probabilistic model over rewards and state transitions within the PSRL algorithm.
Based on the~theoretical insights, we developed Successor Uncertainties, a randomised value function algorithm that avoids some of the pathologies present within previous work. We showed empirically that on hard tabular examples, SU significantly outperforms competing methods, and provided theoretical analysis of its behaviour.
On Atari 2600, we demonstrated Successor Uncertainties is also highly effective when combined with neural network function approximation. 

Performance on the hardest exploration tasks often benefits greatly from multi-step temporal difference learning \citep{precup2000eligibility,munos2016safe,o2017uncertainty}
which we believe is the~most promising direction for improving Successor Uncertainties.
Since modification of existing models to incorporate Successor Uncertainties is relatively simple, other standard techniques used within model-free reinforcement learning like \citep{schaul2015prioritized,wang2015dueling} can be leveraged to obtain further gains.
This paper thus opens many exciting directions for future research which we hope will translate into both further performance improvements and a~more thorough understanding of exploration in modern reinforcement learning.

\subsubsection*{Acknowledgements}

We thank Matej Balog and the~anonymous reviewers for their helpful comments and suggestions. Jiri Hron acknowledges support by a~Nokia CASE Studentship.

\bibliography{neurips}
\bibliographystyle{style/icml2019}

\clearpage
\begin{appendices}
\begin{table}[htp]
    \centering
    \parbox{0.81\textwidth}{\caption{Raw scores for Successor Uncertainties alongside DQN, UBE and Bootstrap DQN . Test conditions: 30 minute emulator time limit and no-ops starting condition. Baselines as reported in \citep{hessel2018rainbow}.}\label{tab:atari-raw-scores}}
    \begin{tabular}{lcccc}
\toprule
Game & DQN & UBE & Bootstrap DQN & SU \\
\midrule
Alien & 1,620.0 & 3,345.3 & 2,436.6 & \textbf{6,924.4} \\
Amidar & 978.0 & 1,400.1 & 1,272.5 & \textbf{1,574.4} \\
Assault & 4,280.4 & \textbf{11,521.5} & 8,047.1 & 3,813.8 \\
Asterix & 4,359.0 & 7,038.5 & 19,713.2 & \textbf{42,762.2} \\
Asteroids & 1,364.5 & 1,159.4 & 1,032.0 & \textbf{2,270.4} \\
Atlantis & 279,987.0 & \textbf{4,648,770.8} & 994,500.0 & 2,026,261.1 \\
Bank Heist & 455.0 & 718.0 & \textbf{1,208.0} & 1,017.4 \\
Battle Zone & 29,900.0 & 19,948.9 & 38,666.7 & \textbf{39,944.4} \\
Beam Rider & 8,627.5 & 6,142.4 & \textbf{23,429.8} & 11,652.3 \\
Bowling & 50.4 & 18.3 & \textbf{60.2} & 38.3 \\
Boxing & 88.0 & 34.2 & 93.2 & \textbf{99.7} \\
Breakout & 385.5 & 617.3 & \textbf{855.0} & 352.7 \\
Centipede & 4,657.7 & 4,324.1 & 4,553.5 & \textbf{7,049.3} \\
Chopper Command & 6,126.0 & 7,130.8 & 4,100.0 & \textbf{15,787.8} \\
Crazy Climber & 110,763.0 & 132,997.5 & 137,925.9 & \textbf{171,991.1} \\
Demon Attack & 12,149.4 & 25,021.1 & 82,610.0 & \textbf{183,243.2} \\
Double Dunk & -6.6 & \textbf{4.7} & 3.0 & -0.2 \\
Enduro & 729.0 & 30.8 & 1,591.0 & \textbf{2,216.3} \\
Fishing Derby & -4.9 & 3.1 & 26.0 & \textbf{53.3} \\
Freeway & 30.8 & 0.0 & \textbf{33.9} & 33.8 \\
Frostbite & 797.4 & 546.0 & 2,181.4 & \textbf{2,733.3} \\
Gopher & 8,777.4 & 13,808.0 & 17,438.4 & \textbf{19,126.2} \\
Gravitar & 473.0 & 224.5 & 286.1 & \textbf{684.4} \\
H.E.R.O. & 20,437.8 & 12,808.8 & 21,021.3 & \textbf{22,050.8} \\
Ice Hockey & -1.9 & -6.6 & \textbf{-1.3} & -2.9 \\
James Bond & 768.5 & 778.4 & 1,663.5 & \textbf{2,171.1} \\
Kangaroo & 7,259.0 & 6,101.2 & 14,862.5 & \textbf{15,751.1} \\
Krull & 8,422.3 & 9,835.9 & 8,627.9 & \textbf{10,103.9} \\
Kung-Fu Master & 26,059.0 & 29,097.1 & 36,733.3 & \textbf{50,878.9} \\
Montezumas Revenge & 0.0 & \textbf{499.1} & 100.0 & 0.0 \\
Ms. Pac-Man & 3,085.6 & 3,141.3 & 2,983.3 & \textbf{4,894.8} \\
Name This Game & 8,207.8 & 4,604.4 & 11,501.1 & \textbf{12,686.7} \\
Pong & 19.5 & 14.2 & 20.9 & \textbf{21.0} \\
Private Eye & 146.7 & -281.1 & \textbf{1,812.5} & 133.3 \\
Q*Bert & 13,117.3 & 16,772.5 & 15,092.7 & \textbf{22,895.8} \\
River Raid & 7,377.6 & 8,732.3 & 12,845.0 & \textbf{17,940.6} \\
Road Runner & 39,544.0 & 56,581.1 & 51,500.0 & \textbf{61,594.4} \\
Robotank & 63.9 & 42.4 & \textbf{66.6} & 58.5 \\
Seaquest & 5,860.6 & 1,880.6 & 9,083.1 & \textbf{68,739.9} \\
Space Invaders & 1,692.3 & 2,032.4 & 2,893.0 & \textbf{13,754.3} \\
Star Gunner & 54,282.0 & 44,458.6 & 55,725.0 & \textbf{78,837.8} \\
Tennis & \textbf{12.2} & 10.2 & 0.0 & -1.0 \\
Time Pilot & 4,870.0 & 5,650.6 & 9,079.4 & \textbf{9,574.4} \\
Tutankham & 68.1 & 218.6 & 214.8 & \textbf{247.7} \\
Up and Down & 9,989.9 & 12,445.9 & 26,231.0 & \textbf{29,993.4} \\
Venture & 163.0 & -14.7 & 212.5 & \textbf{1,422.2} \\
Video Pinball & 196,760.4 & 51,178.2 & \textbf{811,610.0} & 515,601.9 \\
Wizard Of Wor & 2,704.0 & 8,425.5 & 6,804.7 & \textbf{15,023.3} \\
Zaxxon & 5,363.0 & 5,717.9 & 11,491.7 & \textbf{14,757.8} \\
\bottomrule
    \end{tabular}
\end{table}\clearpage

\section{Appendix to section~\ref{sec:rpi}: proofs of propositions~\ref{prop:factorised_symm} and~\ref{prop:prop_uncert_unnecessary}}
\label{a:pu_vs_ic_proofs}

\factorised*
\begin{proof}
    We can w.l.o.g.\ assume that the~distribution is symmetric around zero as centring will not affect validity of the~following argument. 
    To attain probability of taking a~particular action $a$ in state $s$ greater than $\tfrac{1}{2}$, it must be that $\mathbb{P} (a = \argmax_{a'} \hat{Q}(s, a') ) > \tfrac{1}{2}$.
    This event can be described as 
    $$A \coloneqq \bigcap_{a' \in \actions \setminus \{ a \} } \{ \hat{Q} \colon \hat{Q}(s, a) > \hat{Q}(s, a') \} \, ;$$ 
    by symmetry, the~event 
    $$\tilde{A} \coloneqq \bigcap_{a' \in \actions \setminus \{ a \} } \{ \hat{Q} \colon \hat{Q}(s, a) < \hat{Q}(s, a') \} \, ,$$
    must have the~same probability as $A$.
    Because $\mathbb{P}(A) + \mathbb{P}(\tilde{A}) \leq 1$, it must be that $\mathbb{P}(A) \leq \tfrac{1}{2}$.
    Since $\hat{Q}(s, a)$ is by assumption independent of any $\hat{Q}(s', a')$, $(s,a) \neq (s', a')$, the~probability of executing a~sequence of $L$ actions is at best (i.e.\ under deterministic transitions) the~product of probabilities of executing a~single action, which is upper bounded by $2^{-L}$.
\end{proof}

\propunnecessary*
\begin{proof}
    First, let us formally define $G \colon \bar{\bbR}^{\states \times \actions} \to \actions^\states$ to be the~function which maps each Q~function to the~corresponding greedy policy (we can w.l.o.g.\ assume there is some tie-breaking rule for when $\hat{Q}(s, a) = \hat{Q}(s, a'), a \neq a'$, e.g.\ taking the~action with smaller index).
    Here, $\bar{\bbR}$ is the~extended space of real numbers, and we assume the~Borel $\sigma$-algebra generated by the~usual interval topology; the~discrete $\sigma$-algebra is assumed on $\actions$.
    For product spaces, the~product $\sigma$-algebra is taken. 
    Given that the~pre-image of a~particular point $\hat{\pi} \in \actions^\states$ is $\bigcap_{s \in \states} \{ \hat{Q} \colon \hat{Q}( s, \hat{\pi}(s) ) \geq \hat{Q}(s, a) , \forall a \}$, $G$ is measurable and thus the~distribution $P_{\hat{\pi}} = G_{\#} P_{\hat{Q}}$ is well-defined for any $P_{\hat{Q}} \in \mathcal{P}(\bbR^{\states \times \actions})$, and in particular for $\smash{P_{\hat{Q}} = (G \circ F^\pi)_\# P_{\hat{\mathcal{T}}}}$ for any policy $\pi$.
    
    Our proof relies on the~following observation: if we sample $\hat{\pi} \sim P_{\hat{\pi}}$ and then use it to explore the~environment, the~distribution of actions taken in a~particular state $s \in \states$ will be categorical with parameter $p_s \in \{ p \in \bbR_+^{|\actions|} \colon \sum_{j=1}^{|\actions|} p_j = 1 \}$ (except for when the~state $s$ is reached with probability zero under $\smash{P_{\hat{\mathcal{T}}}}$ and $\smash{P_{\hat{\pi}}}$ in which case we can set $p_s$, for example, to $[1 / |\mathcal{A}|, \ldots, 1 / |\mathcal{A}|]^\top$ as this will not affect the~following argument).
    Hence to achieve $G_\# P_{\hat{Q}^\pi} = P_{\hat{\pi}}$, it is sufficient to construct a~model $\hat{Q} \sim P_{\hat{Q}^\pi}$ for which the~distribution of $\argmax_{a \in \actions} \hat{Q}(s, a)$ is categorical with the~parameter $p_s$ for all $s \in \states$. 
    We achieve this using the~Gumbel trick: sample $g_{sa} \sim \mathrm{Gumbel}(0, 1)$ independently for each $(s, a) \in \states \times \actions$, and set $\hat{Q}(s, a) = g_{sa} + \log p_{sa}$ (interpreting $\log 0 = - \infty$).
    
    To finish the~proof, observe that if the~inputs to the~$\argmax$ operator are all shifted by the~same amount, or multiplied by a~positive scalar, the~output remains unchanged.
    Hence taking $\hat{Q}'(s, a) = a + b \hat{Q}(s, a)$ for any $a \in \bbR, b > 0$ will also result in the~desired distribution over exploration policies. We can thus take the~$(s, a)$ for which $\var_{F_\#^\pi P_{\hat{\mathcal{T}}}} [\hat{Q} (s, a)] > 0$ and pick $b > 0$ so that $\var_{P_{\hat{Q}^\pi}} [\hat{Q} (s, a)] \neq \var_{F_\#^\pi P_{\hat{\mathcal{T}}}} [\hat{Q} (s, a)]$ which will be always possible as $\var(b \hat{Q}(s, a))$ is $b^2 \var(g_{sa}) = b^2 \frac{\pi^2}{6}$ if $p_{sa} > 0$ and is undefined otherwise. 
    % The~proposition then follows.
\end{proof}

\clearpage
\section{Appendix to section \ref{sec:experiments}}

\subsection{Proofs for section~\ref{sec:tree-experiments}}\label{a:tree-proofs}

In what follows, the~binary tree MDP of size $L$ introduced in figure~\ref{fig:tree-mdp} is assumed.
% Throughout, we use $p \in \mathbb{N}$ to denote the~highest indexed state $s_p \in \states$ encountered during exploration so far.
We further assume $\smash{\phi}$ is given and maps each state-action to its one-hot embedding.
As all of the~following arguments are independent of the~mapping from the~actions $\{ a_1, a_2 \}$ to the~movements $\{ \aup, \adown \}$, we use $\actions = \{ \aup, \adown \}$ directly for improved clarity. 
% \todo{check that all common assumptions are here (incl.\ repetition of the~assumptions that may have already been stated in the~main text)} 

To prove lemma~\ref{lemma:covariance_condition}, we will need lemmas~\ref{lem:no_cov_down_up} to~\ref{lem:corr_implies_positivity} which we state and prove now.

\begin{lemma}\label{lem:no_cov_down_up}
After any number of posterior updates, the~SU reward distribution is multivariate normal with all rewards mutually independent.
Furthemore, under the~SU Q~function model $\smash{\hat{Q} \sim P_{\hat{Q}^\pi}}$ for any policy $\pi$, and even state indices $0 \leq j < k$ 
\begin{align*}
    &\cov{\qup{k}, \qdown{j}} = \cov{\qdown{k}, \qdown{j}} = 0
    \, .
\end{align*}
\end{lemma}
\begin{proof}
Inspecting equations~\eqref{eq:sf-deriv} and~\eqref{eq:su_prob_dist}, it is easy to see that neither $\qup{k}$ and $\qdown{j}$ nor $\qdown{k}$ and $\qdown{j}$ share any reward terms, since $j < k$ by assumption and the~empirical transition frequencies used to construct $\smash{P_{\hat{Q}^\pi}}$ will always be zero if the~true transition probability is zero (recall that \adown always terminates the~episode).
Hence assuming that the~successor features were successfully learnt, i.e.\ $\hat{\psi}^\pi = \psi^\pi$, it is sufficient to show that the~individual rewards are independent for SU.
To see that this is the~case, observe that the~assumed one-hot encoding of state-actions implies that SU reward distribution will be a~multivariate Gaussian with diagonal covariance after any number of updates which implies the~desired independence.
% If $k > p$ then the~empirical transition probability to state $s_k$ used within the~SU model will be zero and
\end{proof}

\begin{lemma}\label{lem:delta-covariance}
% \todo{fix up the~definition of $\delta$ ($k \mapsto 2k$)}
Under the~SU model $\smash{\hat{Q} \sim P_{\hat{Q}^\pi}}$ for any policy $\pi$, the~random vector $\Delta$, $\Delta_{k/2} \coloneqq \qup{k} - \qdown{k}$, follows a zero mean Gaussian distribution with $\cov{\Delta_{k / 2}, \Delta_{j / 2}} = \cov{\qup{k}, \qup{j}} - \cov{\qdown{k}, \qup{j}})$ for any even indices $0 \leq j < k$.
\end{lemma}
\begin{proof}
The~Gaussianity of the~joint distribution of $\Delta_{j / 2}$ and $\Delta_{ k / 2 }$ follows from the~linearity property of multivariate normal distributions. For the~covariance, observe
\begin{align*}
    \cov{\Delta_{k / 2}, \Delta_{j / 2}} 
    &= 
    \cov{\qup{k} - \qdown{k}, \qup{j} - \qdown{j}} \\
    &=
    \begin{aligned}[t]
        & \cov{\qup{k}, \qup{j}} - \cov{\qdown{k}, \qup{j}} -
        \\
        & \cov{\qup{k}, \qdown{j}} + \cov{\qdown{k}, \qdown{j}}
    \end{aligned}
    \\
     & = \cov{\qup{k}, \qup{j}} - \cov{\qdown{k}, \qup{j}}) \, ,
\end{align*} 
where we used bilinearity of the~covariance operator and then applied lemma~\ref{lem:no_cov_down_up}.
\end{proof}

\begin{lemma}\label{lem:var-condition}
Under the~SU model $\smash{\hat{Q} \sim P_{\hat{Q}^\pi}}$ for the~uniform policy $\pi$, and even indices $0 \leq j < k$
\begin{align*}
    % \phantom{\Longleftrightarrow}
    &&\cov{\qup{k}, \qup{j}} &> \cov{\qdown{k}, \qup{j}}) &&  \\
    \iff
    &&\bbV(\qup{k}) &> \bbV(\qdown{k}) \, . && \phantom{\iff}
\end{align*}
\end{lemma}
\begin{proof}
Analogously to the~proof of lemma~\ref{lem:delta-covariance}, we see that under the~uniform policy
\begin{align*}
    &&&\mathrm{Cov}(\qup{k}, \qup{j})
    \nonumber \\
    &= \,
    &&\mathrm{Cov} (
        \qup{k}, \,
        \rup{j} + 2^{-1} \rup{j + 2} + \ldots + 2^{- (\frac{k - j}{2})} \qup{k}
    )
    \nonumber \\
    &=
    &2^{- (\frac{k - j}{2})} \, &\mathrm{Cov} (\qup{k}, \qup{k})
    =
    2^{- (\frac{k - j}{2})} \, \var(\qup{k})
    \, ,
\end{align*}
where the~$2^{-l}$ terms correspond to the~probability of getting to $s_l$ from $(s_j, \aup)$, $l = 1, 2, \ldots, \frac{k - j}{2}$, and we used bilinearity of the~covariance operator and then applied lemma~\ref{lem:no_cov_down_up}.
An~analogous argument yields $\cov{\qdown{k}, \qup{j}} = 2^{- (\frac{k - j}{2})} \var (\qdown{k})$, concluding the~proof.
\end{proof}

\begin{lemma}\label{lem:corr_implies_positivity}
For a~$d$-dimensional centred Gaussian random vector $\Delta \sim \mathcal{N} (0, \Sigma)$ with $\cov{\Delta_d, \Delta_i} > 0$ for all $i = 1, \ldots, d - 1$, the~following bound holds: $\mathbb{P} ( \Delta_d > 0 \mid \Delta_1 > 0 , \ldots , \Delta_{d-1} >0 ) > 1 / 2$.
% For a~$d$-dimensional centred Gaussian random vector $\Delta \sim \mathcal{N} (0, \Sigma)$ with $\cov{\Delta_d, \Delta_i} > 0$ for all $i = 1, \ldots, d - 1$, the~following bound holds: $\mathbb{E} [ \Delta_d \mid \Delta_1 > 0 , \ldots , \Delta_{d-1} >0 ] > 0$.
\end{lemma}
\begin{proof}
Notice that $\Delta$ and $\Sigma^{1/2} X$, $X \sim \mathcal{N} (0, 1)$, are equal in distribution which allows us to set $\Delta_i = \langle v_i, X \rangle$, with $v_i \in \mathbb{R}^d$ the~$i$\textsuperscript{th} row of $\Sigma^{1/2}$.
Let $R_v \colon \mathbb{R}^d \to \mathbb{R}^d$ be the~reflection against the~orthogonal complement of $v$, i.e.
$$R_v(x) = x - 2 \frac{\langle x , v \rangle}{\langle v, v \rangle} v \, .$$
It is easy to see that $\langle v , R_v (x) \rangle = - \langle v, x \rangle$ and consequently $R_v(R_v(x)) = x$.
The~main idea of this proof is to partition $\mathbb{R}^d$ into the~half-spaces $\{ x \colon \langle v_i , x \rangle > 0 \}$ and $\{ x \colon \langle v_i, R_{v_d} (x) \rangle > 0 \}$, $i = 1, \ldots, d - 1$, and reason about the~value $\langle v_d , x \rangle$ takes in each.

First, we define the~conditioning set $E \coloneqq \{ x \colon \langle v_i, x \rangle  > 0 \, , \forall i = 1, \ldots , d - 1 \}$ and observe that $\mathbb{P} (X \in E) > 0$ so all we need to prove is $\E [ \mathbbm{1}_{\langle v_d , X \rangle > 0} \mathbbm{1}_E  ] > \E [ \mathbbm{1}_{\langle v_d , X \rangle \leq 0} \mathbbm{1}_E ]$, where $\mathbbm{1}_E$ is the~indicator function of the~set $E$.
To do so, we define $U \coloneqq  \{ x \colon \langle v_i, R_{v_d} (x) \rangle > 0 \, , \forall i = 1, \ldots , d - 1 \}$, $A_+ \coloneqq E \, \cap \, U$, $A_- \coloneqq E \, \cap \, U^\mathsf{c}$, split the~integral $\int_{E} \mathbbm{1}_{\{ \langle v_d, X \rangle > 0 \}} (x) \phi (x) \, \mathrm{d} x$ into $\int_{A_+} \mathbbm{1}_{\{ \langle v_d, X \rangle > 0 \}}(x) \phi (x) \, \mathrm{d} x + \int_{A_-} \mathbbm{1}_{\{ \langle v_d, X \rangle > 0 \}}(x) \phi (x) \, \mathrm{d} x$ ($\phi$~is the~standard normal density function; analogously for $\mathbbm{1}_{\{ \langle v_d, X \rangle \leq 0 \}}$), and consider $X \in A_+$ and $X \in A_-$ separately:

{(I)}~$X \in A_+$: Take any $x , v \in \mathbb{R}^d$ and define the~orthogonal projection map on $v$, $B_v \coloneqq v v^\top / \| v \|_2^2$, and the~corresponding projections of $x$, $x_{v} \coloneqq B_v x \, , x_{v}^\bot = (I - B_v) x$, so that $x = x_{v} + x_{v}^\bot$.
Since
$$\| x \|_2^2 = \| x_v + x_v^\bot \|_2^2 = \| x_v \|_2^2 + \| x_v^\bot \|_2^2 = \| -x_v + x_v^\bot \|_2^2 = \| R_v (x) \|_2^2 \, ,$$
it follows that $\phi(x) = \phi(R_{v_d}(x))$.
Noticing further that $R_{v_d}(x) = (I - 2 B_{v_d}) x$ and recalling $R_{v_d} (R_{v_d} (x)) = x$, we have $|\det \nabla_x R_{v_d} (x)| = |-1| = 1$.
The~crucial observation here is $\langle x , v_d \rangle > 0 \iff \langle x_{v_d}, v_d \rangle > 0$, $\langle x, v_d \rangle \leq 0 \iff \langle R_{v_d} (x) , v_d \rangle > 0$ (up to null sets), and that $A_+ = R_{v_d} [ A_+ ] = \{ R_{v_d}(x) \colon x \in A_+ \}$ which follows from the~definition of the~set $A_+$.
In particular this means that whenever $x \in A_+$ then also $-x \in A_+$, and thus by the~above established symmetry and the~change of variable formula, 
$\int_{A_+} \mathbbm{1}_{\{ \langle v_d, X \rangle > 0 \}} (x) \phi(x) \, \mathrm{d} x = \int_{A_+} \mathbbm{1}_{\{ \langle v_d, X \rangle \leq 0 \}}(x) \phi(x) \, \mathrm{d} x$, i.e.~the~conditional probabilities of both $A_+ \, \cap \, \{ \langle v_d, X \rangle > 0 \}$ and $A_+ \, \cap \, \{ \langle v_d, X \rangle \leq 0 \}$ are equal.

(II)~$X \in A_-$: Notice that for any $i = 1 , \ldots , d - 1$
$$\langle v_i , R_{v_d} (x) \rangle = \langle v_i , x \rangle - 2 \frac{\langle v_d, x \rangle}{\| v_d \|_2^2} \langle v_d, v_i \rangle \, .$$
Hence if $\langle v_d, x \rangle \leq 0$ then $\langle v_i , R_{v_d} (x) \rangle \geq \langle v_i , x \rangle > 0$ from the~definition $\langle v_d , v_i \rangle = \cov{\Delta_d, \Delta_i}$ and the~assumption $\cov{\Delta_d, \Delta_i} > 0$.
Now by the~definition of $U$ in $A_- = E \, \cap \, U^{\mathsf{c}}$, for any $x \in A_-$, there must exist $i \in \{ 1 , \ldots , d - 1 \}$ such that $\langle v_i , R_{v_d} (x) \rangle \leq 0$ which implies $\langle v_d , x \rangle > 0$ by the~above argument.
It is thus sufficient to establish $\mathbb{P} (X \in A_-) > 0$ to complete the~proof as the~intersection $A_- \, \cap \, \{ \langle v_d, X \rangle \leq 0 \}$ is empty.

Since $\langle v_d, v_i \rangle = \cov{\Delta_d, \Delta_i} > 0$, $v_d \in E$ and $\langle v_i , R_{v_d} (v_d) \rangle = - \langle v_i , v_d \rangle < 0 \, , \forall i = 1, \ldots, d - 1$, we have $v_d \in A_-$.
We can thus construct a~convex polytope $V \subseteq A_-$ such that $\mathbb{P}(X \in V) > 0$. 
Specifically, pick some $i \in \{ 1, \ldots, d - 1 \}$, for example $i = \argmax_{i \in \{1, \ldots, d-1 \}} \langle v_d, v_i \rangle$, and set $\kappa \coloneqq \max_{k, l \in \{1, \ldots , d \}} |\langle v_k, v_l \rangle| = \max_{k \in \{ 1, \ldots, d\} } \| v_k \|_2^2  > 0$.
Now define
$$V \coloneqq  \{x \colon x = u + v_d + \sum_{j = 1}^{d-1} \alpha_j v_j \, , \alpha_j \in [0,  \tfrac{\langle v_d, v_i \rangle}{\kappa (d - 1)} ) \, , u \in \mathrm{span} (v_1, \ldots, v_d)^\bot \} \, ,$$
where $\mathrm{span} (v_1, \ldots, v_d)^\bot$ is the~orthogonal complement of the~linear span of the~vectors $(v_1, \ldots , v_d)$.
Clearly $V \subseteq E$ as for any $x \in V$, $\langle v_i , x \rangle > 0$ from the~bound on the~coefficients $\alpha$.
To see that $x \in V \implies x \in U^\mathsf{c}$, note
\begin{equation*}
    \langle v_i, R_{v_d} (x) \rangle
    =
    -
    \langle v_i, v_d \rangle
    +
    \sum_{j=1}^{d-1}
        \underbrace{\alpha_j}_{\geq 0} \biggl[
            \langle v_i , v_j \rangle
            -
            2 \underbrace{
                \frac{\langle v_d , v_i \rangle}{\| v_d \|_2^2}
                \langle v_j , v_d \rangle
            }_{> 0}
        \biggr]
    \, .
\end{equation*}
Since the~first and last terms are strictly negative, we just need to control the~second term.
We again apply the~definition of $V$ to bound $\sum_j \alpha_j \langle v_i, v_j \rangle < \langle v_i, v_d \rangle$ which implies $\langle v_i, R_{v_d}(x) \rangle < 0$ for every $x \in V$.
Thus $V \subseteq A_-$ and because $V$ has non-zero volume, its probability under $\mathcal{N}(0 , I)$ will be positive.
Hence $\int_{A_-} \mathbbm{1}_{\{ \langle v_d, X \rangle > 0 \}} (x) \phi(x) \, \mathrm{d} x > \int_{A_-} \mathbbm{1}_{\{ \langle v_d, X \rangle \leq 0 \}} \phi(x) \, \mathrm{d} x = 0$.
\end{proof}

We are now ready to prove lemma~\ref{lemma:covariance_condition}.
\newcounter{lemmabkup}
\setcounter{lemmabkup}{\value{lemma}}
\setcounter{lemma}{\value{suboundcounter}}
\stepcounter{lemma}
\begin{lemma}[Formal statement]
Let $\smash{\hat{\pi} \sim P_{\hat{\pi}} = G_\# P_{\hat{Q}^\pi}}$ where $\smash{\hat{Q} \sim P_{\hat{Q}^\pi}}$ is the~SU model for the~uniform policy $\pi$.
For $2 \leq k < 2L$ even, define $\smash{U_k = \{ \hat{\pi} \colon \hat{\pi}(s_0) = \ldots = \hat{\pi}(s_{k - 2}) = \delta_{\aup} \}}$ where $\smash{\delta_{\aup}}$ is the~policy of selecting $\aup$ with probability one.
Then $\smash{P_{\hat{\pi}} (\hat{\pi} (s_k) = \delta_{\aup} \mid \hat{\pi} \in U_k) > 1 / 2 }$ if there exists an~even $0 \leq j < k$ such that 
$
    \cov{\qup{k}, \qup{j}} > \cov{\qdown{k}, \qup{j}}
    \, .
$ 
\end{lemma}
\setcounter{lemma}{\value{lemmabkup}}

\begin{proof} % [Proof of lemma~\ref{lemma:covariance_condition}]
Under $P_{\hat{\pi}}$, $G(\hat{Q}) = \delta_{\aup}$ iff $\Delta_{k/2} = \hat{Q}(s_k , \aup) - \hat{Q}(s_k, \adown) > 0$. 
By lemma~\ref{lem:delta-covariance}, the~distribution of the~random vector $\Delta = [\Delta_0, \Delta_1, \ldots , \Delta_{k/2}]^\top$ is a~zero mean Gaussian, and in particular
$$P_{\hat{\pi}} (\hat{\pi} = \delta_\aup \mid \hat{\pi} \in U_k) = \mathbb{P} (\Delta_{k / 2} > 0 \mid \Delta_{0} > 0 , \ldots , \Delta_{k / 2 - 1} > 0 )  \, .$$
To prove the~desired claim, we therefore need to show that existence of even $0 \leq j < k$ such that
$\smash{
    \cov{\qup{k}, \qup{j}} > \cov{\qdown{k}, \qup{j}}
    \, ,
}$ 
implies $\mathbb{P} (\Delta_{k / 2} > 0 \mid \Delta_{0} > 0 , \ldots , \Delta_{k / 2 - 1} > 0 ) > 1 / 2$.
The~statement follows from:
\begin{align*}
    & \cov{\qup{k}, \qup{j}} > \cov{\qdown{k}, \qup{j}} \, \text{, for some even } 0 \leq j < k \\
    & \overset{\text{lemma~\ref{lem:var-condition}}}{\iff}
    \cov{\qup{k}, \qup{j}} > \cov{\qdown{k}, \qup{j}} \, \text{, for all even } 0 \leq j < k \\
    & \overset{\text{lemma~\ref{lem:delta-covariance}}}{\iff}
    \cov{\Delta_{k/2}, \Delta_{j / 2}} > 0 \, \text{, for all even } 0 \leq j < k \\
    & \overset{\text{lemma~\ref{lem:corr_implies_positivity}}}{\iff}
    \mathbb{P} (\Delta_{k / 2} > 0 \mid \Delta_{0} > 0 , \ldots , \Delta_{k / 2 - 1} > 0 ) > 1 / 2
    \, .
\end{align*}
% Writing $\Delta_k = \qup{k} - \qdown{k}$, under posterior sampling the event $(s_k, \textsc{up})$ occurs if $\Delta_k > 0$. Given that we are in the state $s_k$, we have $\Delta_i > 0$ for $i < k$ almost surely. Noting conditionals of Gaussian random variables are Gaussian, we see that $\Delta_k \mid \Delta_1, \dots, \Delta_{k-1} \sim \mathcal{N}(\mu_{k \mid <k}, \sigma_{k \mid <k})$ for some $\mu_{k \mid <k} = \bbE[\Delta_k \mid \Delta_1 > 0, \dots, \Delta_{k-1} > 0]$ and $\sigma_{k \mid <k} > 0$. Therefore
% \begin{align*}
%     \pi(s_k)(\aup)& > \pi(s_k)(\aup) \\ 
%     \iff\ & P(\Delta_k > 0 \mid \Delta_1 > 0, \dots, \Delta_{k-1} > 0) = \Phi\left(\frac{\mu_{k\mid <k}}{\sigma_{k\mid <k}}\right) > \frac{1}{2} \\ 
%     \iff\ & \bbE[\Delta_k \mid \Delta_1 > 0, \dots, \Delta_{k-1} > 0] > 0 \ \mathrm{a.s.} \\
%     \iff\ & (\forall j < i \leq k)\ (\cov{\Delta_j, \Delta_i} > 0) \\
%     \stackrel{\mathrm{lemma~\ref{lem:delta-covariance}}}{\iff}& (\forall j < i \leq k)\ (\cov{\qup{i}, \qup{j}} > \cov{\qdown{i}, \qup{j}}) \\
%     \stackrel{\mathrm{lemma~\ref{lem:var-condition}}}{\iff}& (\forall i \leq k)\ ( \bbV(\qup{i}) - \bbV(\qdown{i}) > 0) \\
%     \iff\ & \ (\bbV(\qup{k}) - \bbV(\qdown{k}) > 0) \\
%     \stackrel{\mathrm{lemma~\ref{lem:var-condition}}}{\iff}& (\exists j < k)\ (\cov{\qup{k}, \qup{j}} > \cov{\qdown{k}, \qup{j}})
% \end{align*}
\end{proof}

% \begin{proposition}[Informal statement] %\label{prop:su-bound}
% Assume the~SU model with: (i)~fixed one-hot state-action embeddings $\phi$, (ii)~uniform exploration thus far, (iii)~successor representations learnt to convergence for a uniform policy.
% Let $s_{k}$ for $2 \leq k < 2L$ be a state visited $N$ times thus far.
% Then the~probability of selecting $\aup$ in $s_k$, given $\aup$ was selected in $s_0 , s_2, \ldots , s_{k - 2}$, is greater than one half with probability greater than $1 - \epsilon_N$, where $\epsilon_N$ decreases exponentially with $N$. 
% \end{proposition}

\setcounter{proposition}{\value{suboundcounter}}
\begin{proposition}[Formal statement]
Assume the~SU model with: (i)~one-hot state-action embeddings~$\phi$, (ii)~uniform exploration thus far, (iii)~successor representations learnt to convergence for a uniform policy.
For $2 \leq k < 2L$ even, let $s_{k}$ be a state visited $N$ times thus far, and $\pi$, $\smash{\hat{Q} \sim P_{\hat{Q}^\pi}}$, $\smash{\hat{\pi} \sim P_{\hat{\pi}}}$, and $U_k$ be defined as in lemma~\ref{lemma:covariance_condition}.
% Let $\smash{P}_{\hat{\pi}}$ be the~distribution over policies $G_\# P_{\hat{Q}^\pi}$, $2 \leq k < 2L$ be an~even of index $s_k$ visited $N$ times thus far, and define $\smash{U_k = \{ \hat{\pi} \colon \hat{\pi}(s_0) = \ldots = \hat{\pi}(s_{k - 2}) = \delta_{\aup} \}}$ where $\smash{\delta_{\aup}}$ is the~policy of selecting $\aup$ with probability one.
Then 
$$\smash{P_{\hat{\pi}} (\hat{\pi} (s_k) = \delta_{\aup} \mid \hat{\pi} \in U_k)  > P_{\hat{\pi}} (\hat{\pi} (s_k) = \delta_{\adown} \mid \hat{\pi} \in U_k) }\, , $$
with probability greater than $1 - \epsilon_N$, where $\epsilon_N < 0.75^N e^{-\frac{N}{50}} + (1 - 0.75^N)e^{- 0.175N}$.
\end{proposition}

\begin{proof}
By lemma~\ref{lemma:covariance_condition}, we know that
$\smash{P_{\hat{\pi}} (\hat{\pi} (s_k) = \delta_{\aup} \mid \hat{\pi} \in U_k)  > P_{\hat{\pi}} (\hat{\pi} (s_k) = \delta_{\adown} \mid \hat{\pi} \in U_k) }$ holds if $\cov{\qup{k}, \qup{j}} > \allowdisplaybreaks \cov{\qdown{k}, \qup{j}}$ for some $j = 0, 2, \ldots, k - 2$.
By lemma~\ref{lem:var-condition}, this condition is equivalent to requiring $\bbV(\qup{k}) > \bbV(\qdown{k})$.
Our approach is thus based on lower bounding the~probability of the~event
\begin{equation}\label{eq:var_cond}
    \{ \hat{Q} \colon \bbV(\qup{k}) > \bbV(\qdown{k}) \} \, . 
\end{equation}
The~rest of the~proof is divided into two stages:
\begin{enumerate}[label=(\Roman*),itemsep=2pt,topsep=0pt]
    \item We derive a~crude bound $\smash{\Upsilon}_1(\qup{k}) \leq \bbV(\qup{k})$ and compute a lower bound on the probability of the~event $\smash{\Upsilon}_1(\qup{k}) > \bbV(\qdown{k})$.
    \item We then derive a tighter lower bound $\smash{\Upsilon}_2(\qup{k})$, and again compute a lower bound on the probability of the~event $\smash{\Upsilon}_2(\qdown{k}) > \bbV(\qdown{k})$.
\end{enumerate}

{(I)}~The bound $\smash{\Upsilon_1(\qup{k}) \leq \bbV(\qup{k})}$ will correspond to a worst case assumption about the distribution of data available from exploration, and $\smash{\Upsilon_2(\qup{k})}$ to a~less pessimistic scenario.
The~change of setup involved in moving from the first bound to the second will be illustrative of the~manner in which, under the SU model, the more states the agent has previously observed beyond $s_k$, the more likely it is to satisfy the condition from equation~\eqref{eq:var_cond} and consequently $\cov{\qup{k}, \qup{j}} > \allowdisplaybreaks \cov{\qdown{k}, \qup{j}}$ for all $j = 0, 2, \ldots, k - 2$.

From lemma~\ref{lem:no_cov_down_up}, we know that the~SU model of rewards will be a~zero mean Gaussian with a~diagonal covariance. In particular, the~covariance takes the~form $(\theta^{-1} I + \beta^{-1} \sum_t \phi_t \phi_t^\top )^{-1}$, where recall $\theta$ is the prior and $\beta$ is the likelihood variance, implying that the~diagonal entries will be $\nu(n) \coloneqq (\theta^{-1} + \beta^{-1} n)^{-1}$ where $n$ is the number of times the corresponding state-action was observed.

% To derive the~bound $\smash{\Upsilon_1(\qup{k}) \leq \bbV(\qup{k})}$. 
Recall that the~agent has previously visited the~state $s_k$ $N$ times. 
We will write $\smash{N_1}$ for the~number of times we have observed $(s_k, \aup)$ so far, $\smash{N_2}$ for the~number of times $(s_k, \aup)$ \emph{and} $(s_{k + 2}, \aup)$ have both been observed within a~single episode, and so forth.  
% \todo{the need for this definition is suspect}
Observe
\begin{align*}
    \bbV(\qup{k}) &= 
    \begin{aligned}[t]
        &\nu(N_1) + 2^{-1}(\nu(N_2) + \nu(N_1 - N_2)) + \\
        &\mathbbm{1}_{N_3 > 0} 2^{-2}(\nu(N_3) + \nu(N_2 - N_3) + \mathbbm{1}_{N_4>0}\ldots)
    \end{aligned}
    \\ 
    & \geq \nu(N_1) + 2^{-1}(\nu(N_2) + \nu(N_1 - N_2))
\end{align*}
We now minimise $\nu(N_2) + \nu(N_1 - N_2)$ with respect to $N_2$, finding the minima to occur at $N_2 = N_1$ and $N_2 = 0$, in both cases giving the bound 
\begin{align*}
    \Upsilon_1(\qup{k}) \coloneqq \frac{3}{2}\nu(N_1) + \frac{1}{2}\nu(0) \leq \bbV(\qup{k})
\end{align*}
This bound can be interpreted as assuming that after taking action $\aup$, the agent has always proceeded to move \adown, thus terminating the episode.
We now compute a lower bound on the probability that $\Upsilon_1(\qup{k}) > \bbV(\qdown{k})$, in terms of $N_1$. 
We have
\begin{align*}
    &\Upsilon_1(\qup{k}) - \bbV(\qdown{k}) = \frac{3}{2}\nu(N_1) - \nu(N - N_1) + \frac{1}{2}\nu(0) 
    >\frac{3}{2} \nu(N_1) - \nu(N-N_1)
\end{align*}
which is greater than zero when $\theta^{-1} + \beta^{-1}(3N - 5N_1) > \beta^{-1}(3N - 5N_1) > 0$, i.e.\ whenever $\smash{N_1 < \frac{3N}{5}}$.
By Hoeffding's inequality, $\mathbb{P}(N_1 \geq \frac{(1+\delta)N}{2}) \leq e^{-\frac{\delta^2 N}{2}}$.
Thus, letting $\delta= 5^{-1}$, $\smash{\bbV(\qup{k}) > \bbV(\qdown{k})}$ holds with probability greater than $1 - e^{-\frac{N}{50}}$.

{(II)}~Notice that we have obtained the~$\Upsilon_1$ bound by considering the~worst case scenario for $N_2$, namely $N_2 = 0$.
Here we derive a~tighter bound by treating the~two cases, $N_2 = 0$ and $N_2 > 0$, separately.
For $N_2 > 0$, we follow an~approach analogous to (I): we assume the~``next'' worst-case scenario, which is easily seen to be $N_3 = 0$, and compute a~lower bound on $\smash{\bbV(\qup{k})}$
\begin{equation*}
    \Upsilon_2(\qup{k}) \coloneqq \nu(N_1) + \nu(N_2) + \frac{1}{2}\nu(N_1 - N_2) \, .
\end{equation*}
After some algebra, we obtain $\smash{\Upsilon_2(\qup{k}) >  \bbV(\qdown{k})}$ for all $N_2 > 0$ and $N_1 \leq \frac{1}{41}(27 + 4\sqrt{2})N \eqqcolon c \, .$
We thus only need to bound the~probability of $N_1 > c$.
Using Hoeffding's inequality as in (I) for a~suitably chosen $\delta$, we see $\mathbb{P}(N_1 > c) \leq \exp\{-\frac{(13 + 8\sqrt{2})^2}{3362}N\} < e^{-0.175N}$.
For $N_2 = 0$, we use the~bound from part (I), and thus the~only thing remaining is to compute the~probability of $N_2 = 0$:
\begin{align*}
    \mathbb{P}(N_2 = 0) &= \textstyle{\sum}_{K=0}^N \mathbb{P}(N_2 = 0 \mid N_1 = K) \mathbb{P}(N_1 = K) = \textstyle{\sum}_{K=0}^N \, 2^{-K} 2^{-N} \binom{N}{K} \\
    &= \textstyle{\sum}_{K=0}^N \binom{N}{K} \, 4^{-K} 2^{K-N} = (4^{-1} + 2^{-1})^N = 0.75^N
    \, .
\end{align*}
Combining the~above results, we see that $\smash{\bbV(\qup{k}) > \bbV(\qdown{k})}$ will hold with probability greater than $1 - \epsilon_N$ where $\epsilon_N < 0.75^N e^{-\frac{N}{50}} + (1 - 0.75^N)e^{- 0.175N}$.
% Instead, we proceed to compute the probability of the event $N_2 = 0$, and account for it explicitly within the full bound. This can be computed as follows
% \begin{align*}
%     \mathbb{P}(N_2 = 0) &= \textstyle{\sum}_{K=0}^N \mathbb{P}(N_2 = 0 \mid N_1 = K) \mathbb{P}(N_1 = K) = \textstyle{\sum}_{K=0}^N \, 2^{-K} 2^{-N} \binom{N}{K} \\
%     &= \textstyle{\sum}_{K=0}^N \binom{N}{K} \, 4^{-K} 2^{K-N} = (4^{-1} + 2^{-1})^N = 0.75^N
% \end{align*}
% We now derive $\Upsilon_2(\qup{k})$, a tighter bound that will hold with probability $1 - 0.75^N$.
% We proceed as before, except now we assume that $N_2 > 0$ and that thereafter $N_3 = 0$. That is, the agent fell into the next worst-case scenario, and has always gone \adown after each $(\aup, \aup)$ sequence. Following the same steps as before, this provides the bound 
% \begin{equation*}
%     \Upsilon_2(\qup{k}) \coloneqq \nu(N_1) + \nu(N_2) + \frac{1}{2}\nu(N_1 - N_2),
% \end{equation*}
% and it can be easily verified that $\smash{\hat{\bbV}_2(\qup{k}) >  \bbV(\qdown{k})}$ for all $N_2 > 0$ and $N_1 \leq \frac{1}{41}(27 + 4\sqrt{2})N$. Using Hoeffding's inequality for $N_1$, we find that $\mathbb{P}(N_1 \geq \frac{1}{41}(27 + 4\sqrt{2})N) \leq \exp{\{-\frac{(13 + 8\sqrt{2})^2}{3362}N\}} < e^{-0.175N}$. Combining the bounds derived from $\Upsilon_1(\qdown{k})$ and $\Upsilon_2(\qdown{k})$ weighted by the respective probabilities that these hold we obtain the statement of the proposition.
\end{proof}

\subsection{Proofs for section~\ref{sec:tied-actions}}\label{a:tied-action-proofs}
The~following is an~extension of proposition~\ref{prop:bdqn-wins2} to activations such as ReLU, Leaky ReLU, or Tanh.

\begin{prop}
\label{prop:bdqn-wins}
Consider the~same setting as in proposition~\ref{prop:bdqn-wins2} with the~exception that $\varphi$ for which $\varphi [ (0 , \infty) ] = \{\varphi(x) : x > 0\} \subseteq (0, \infty)$.
Then sampling independently form the~prior $w_a \sim \mathcal{N} (0, \sigma_w^2 I)$, $U_{hs} \sim \mathcal{N} (0, \sigma_u^2)$ solves a~tied action binary tree of size $L$ in $T \leq - [ \log_2 (1 - 2^{-d}(1 - 2^{-d})^L ) ]^{-1}$ median number of episodes, or approximately ${- [\log_2 (1 - 2^{-d}) ]^{-1}}$ for $d \geq 10$.
\end{prop}

\begin{proof}
As in the~proof of proposition~\ref{prop:bdqn-wins2}, let us define $\Delta \coloneqq w_{\aup} - w_{\adown}$ and observe $\aup$ is selected if $\smash{\hat{Q}(s, \aup) - \hat{Q}(s , \adown)} = \langle \phi(s) , w_{\aup} - w_{\adown} \rangle > 0$. We can thus lower bound
\begin{equation*}
    \mathbb{P}\bigl[\bigcap_{j=0}^{L-1} \{ \smash{\hat{Q}(s_{2j}, \aup)} \! > \! \smash{ \hat{Q}(s_{2j}, \adown)} \} \bigr] \!\geq\! \mathbb{P}\bigl[ \bigcap_{j=0}^{L-1} \{ \langle \phi(s_{2j}), \Delta \rangle \! > \! 0 \} \mid \Delta \! > \! 0\bigr] \mathbb{P}(\Delta \! > \! 0) \, ,
\end{equation*}
where $\Delta > 0$ is meant elementwise. 
As $\Delta \sim \mathcal{N}(0, 2\sigma_w^2 I)$, $\mathbb{P}(\Delta > 0) = 2^{-d}$ for all $L$.
By independence $\mathbb{P}\bigl[ \bigcap_{j=0}^{L-1} \{ \langle \phi(s_{2j}), \Delta \rangle \! > \! 0 \} \mid \Delta \! > \! 0\bigr] = \prod_{j=0}^{L-1} \mathbb{P} (\{ \phi(s_{2j}) > 0 \})$ where $>$ is to be interpreted elementwise.
From the~assumption $\varphi [(0, \infty) ] \subseteq (0, \infty)$ and the~assumed $\phi(s) = \varphi(U 1_s)$, $U_{hs} \sim \mathcal{N}(0, \sigma_u^2)$, we have $\mathbb{P} (\{ \phi(s) > 0 \}) \geq 1 - 2^{-d}$, which implies that probability of success within a~single episode is lower bounded by $2^{-d} (1 - 2^{-d})^L$. The~result follows by substituting this probability into the~formula for the~median of a~geometric distribution.
\end{proof}

\section{Appendix to section~\ref{sec:experiments}: implementation \& experimental details}\label{a:implementation}
Pseudocode for SU. Quantities superscripted with~$\smash{\dagger}$ are treated as fixed during optimisation.

\begin{algorithm}
    \caption{Successor Uncertainties with posterior sampling}
    \label{alg:su-algorithm}
    \begin{algorithmic}
        \Require{Neural networks $\hat{\psi}$ and $\hat{\phi}$; weight vector $\hat{w}$; prior variance $\theta > 0$; likelihood variance $\beta > 0$; covariance decay factor $\zeta \in [0, 1]$; $\textsc{batch\_size}\in \mathbb{N}$; $\textsc{learning\_rate} > 0$; environment $\textsc{Env}$; action set $\actions$; discount factor $\gamma \in [0, 1)$.}
        \vspace{0.5\baselineskip}
        % \vspace{0.1\baselineskip}
        \State initialise $\Lambda \gets \theta^{-1} I$, $\hat{\Sigma}_w \gets \Lambda^{-1}$ 
        \For{\textbf{each} episode}
            \State sample $w \sim N(\hat{w}, \hat{\Sigma}_w)$ 
            \State $s \gets \textsc{env.reset}()$ 
            \vspace{0.5\baselineskip}
            \Repeat
                \State $a \gets \argmax_{z \in \actions} \langle \hat{\psi}(s, z), w \rangle$
                \State $s^\prime, r, done \gets \textsc{env.interact}(s)$
                
                \State $\mathcal{D} \gets \mathcal{D} \cup \{(s, a, r, s^\prime, done)\}$
                \vspace{0.5\baselineskip}
                
                \State $\mathcal{B} \sim \textsc{Uniform}(\mathcal{D}, \textsc{batch\_size})$
                \State $\ell \gets \sum_{b \in \mathcal{B}} \textsc{SU\_Loss}(b, \hat{\Sigma}_w) $
                
                \State $\hat{\phi},\hat{\psi},\hat{w} \gets \textsc{SGD.step}(\ell, \textsc{learning\_rate})$ 
                \vspace{0.5\baselineskip}

                \State $\Lambda \gets \zeta \Lambda + \beta^{-1}\hat{\phi}(s, a)\hat{\phi}(s, a)^{\top}$
                \State $s \gets s^\prime$
            \Until{$done$}
            \vspace{0.5\baselineskip}
            \State $\hat{\Sigma}_w \gets \Lambda^{-1}$
        \EndFor
        \\
        \Function{\textsc{SU\_Loss}}{\textsc{experience\_tuple}, $\hat{\Sigma}_w$}
            \State $s, a, r, s, done \gets \textsc{experience\_tuple}$
            \State sample $w \sim N(\hat{w}, \hat{\Sigma}_w)$
            \State $a^\prime \gets \argmax_{z \in \actions} \langle \hat{\psi}(s, z), w \rangle$ 
            \vspace{0.5\baselineskip}
            \State $y_Q \gets 
                \begin{cases}
                0 
                & \mathrm{if}\ done \\ 
                \gamma\langle \hat{w}, \hat{\psi}(s^\prime, a^\prime) \rangle 
                & \mathrm{otherwise}
                \end{cases}$
            \State $y_{SF} \gets 
                \begin{cases} 
                0  
                & \mathrm{if}\ done \\ 
                \gamma\hat{\psi}(s^\prime, a^\prime) 
                & \mathrm{otherwise}
                \end{cases}$
            \vspace{0.5\baselineskip}
            \State \Return $|\langle \hat{w}, \hat{\phi}(s, a) \rangle - r|^2 + \| \hat{\psi}(s, a) - \hat{\phi}(s, a) - y_{SF}^\dagger \|_2^2 + |\langle \hat{w}, \hat{\psi}(s, a) \rangle - r - y_Q^\dagger|^2$
        \EndFunction
    \end{algorithmic}
\end{algorithm}
\subsection{Appendix to sections~\ref{sec:tree-experiments} and \ref{sec:chain-experiments}: tabular experiments}\label{a:tree-experiments}\label{a:chain-experiments}

\paragraph{Neural network architecture} The architecture used for tabular experiments consists of:
\begin{enumerate}[itemsep=2pt,topsep=0pt]
    \item A neural network mapping one-hot encoded state vectors and one-hot encoded action vectors to a hidden layer $\hat{\phi}(s, a)$, and then to reward prediction $\hat{r}(s, a)$ via weights $\hat{w}$. Weights mapping state vectors to hidden layer are initialised using a folded Xavier normal initialisation and followed by ReLU activation. Weights $\hat{w}$ are initialised to zero, consistent with a Bayesian linear regression model with a zero mean prior.
    \item A set of weights that linearly maps state-action vectors to $\hat{\psi}(s, a)$.
\end{enumerate}

\paragraph{Binary tree MDP} Table~\ref{tab:tree-parameters} contains the hyperparameters considered during gridsearch and the final values used to produce figure~\ref{fig:scaling}. Hyperparameter values are not included for UBE and BDQN, as they do not affect performance (that is, BDQN and UBE perform uniformly random exploration for all hyperparameter settings). All methods used one layer fully connected ReLU networks, Xavier initialisation, and a replay buffer of size 10,000. Hyperparameters for all methods were selected by gridsearch on a $L=100$ sized binary tree. Hyperparameters were then kept fixed as binary tree size $L$ was varied.

\begin{table}[h]
    \centering
    \caption{\label{tab:tree-parameters} Binary tree experiment algorithm hyperparameters gridsearch sets and values used for Successor Uncertainties, Bootstrap+Prior (1x compute) and Bootstrap+Prior (25x compute).}
    \begin{tabular}{lcccc}
            \toprule
            & & \multicolumn{3}{c}{Algorithm} \\ \cmidrule(lr){3-5}
            Hyperparameter & Gridsearch set & SU & B+P 1x & B+P 25x \\ 
            \midrule
            Gradient steps per episode & --- & 10 & 10 & 250 \\
            Hidden size & $\{ 20, 40 \}$ & 20 & 20 & 20 \\
            \rule{0pt}{14pt}Prior variance $\theta$ & $\{1, 10^2, 10^4 \}$ & $10^4$ & --- & --- \\
            Likelihood variance $\beta$ & $\{10^{-3}, 10^{-2}, 10^{-1} \}$ & $10^{-3}$ & --- & --- \\
            $\hat{\Sigma}_w$ decay factor $\zeta$ & --- & 1 & --- & --- \\
            \rule{0pt}{14pt}Ensemble size $K$ & $\{ 10, 20, 40\}$ & --- & $10$ & $10$ \\
            Bootstrap probability & $\{0.1, 0.25, 0.75, 0.9, 1.0\}$  & --- & 0.75 & 1.0 \\
            Prior weight & $\{0.0, 0.1, 1.0, 10.0 \} $ & --- & 0.1 & 0.0 \\
            \bottomrule
    \end{tabular}
\end{table}

\paragraph{Chain MDP} Problem description copied verbatim from \citet{osband2018randomized}: 
\begin{quotation}
\emph{
\noindent The environments are indexed by problem scale $L \in \mathbb{N}$ and action mask $W \sim {\rm Ber}(0.5)^{L \times L}$, with $\states = \{0, 1\}^{L \times L}$ and $\actions = \{0,1\}$.
The agent begins each episode in the upper left-most state in the grid and deterministically falls one row per time step.
The state encodes the agent's row and column as a one-hot vector $s_t \in \states$.
The actions $\{0, 1\}$ move the agent left or right depending on the action mask $W$ at state $s_t$, which remains fixed.
The agent incurs a cost of $0.01/L$ for moving right in all states except for the right-most, in which the reward is $1$.
The reward for action left is always zero.
An episode ends after $L$ time steps so that the optimal policy is to move right each step and receive a total return of $0.99$; all other policies receive zero or negative return.
}
\end{quotation}

Table~\ref{tab:chain-parameters} contains the hyperparameter settings used to produce the results in figure~\ref{fig:scaling-log-log}. We were unable to run experiments with $L>160$ for Successor Uncertainties due to memory limitations. $|\states|$ scales as $\mathcal{O}(L^2)$ for this problem. Consequently, with one hot encoding, the required neural network weight vectors required grew too large. A smarter implementation using a library designed for operating on sparse embeddings would alleviate this problem.

\begin{table}[h]
    \centering
    \caption{\label{tab:chain-parameters} Hyperparameters used for Successor Uncertainties in chain experiments. Hidden size fixed at 20 to match architecture in \citet{osband2018randomized}.}
    \begin{tabular}{lcc}
            \toprule
            Hyperparameter & Gridsearch set & Value used \\
            \midrule
            Gradient steps per episode & $\{10, 20, 40 \}$ & 40 \\
            Hidden size & --- & 20 \\
            \rule{0pt}{14pt}Prior variance $\theta$ & $\{ 1, 10^2, 10^4 \}$ & $1$ \\
            Likelihood variance $\beta$ & $\{10^{-3}, 10^{-2}, 10^{-1} \}$ & $10^{-2}$ \\
            $\hat{\Sigma}_w$ decay factor $\zeta$ & --- & 1 \\
            \bottomrule
    \end{tabular}
\end{table}

\subsection{Appendix to section~\ref{sec:atari-experiments}: Atari 2600 experiments}\label{a:atari-experiments}

\paragraph{Training procedure} We train for 200M frames (50M action selections with each action repeated for 4 frames), using the \textsc{Adam} optimiser \citep{kingma2014adam} with a learning rate of $5\times 10^{-5}$ and a batch size of $32$. A target network is utilised, as in \citet{mnih2015human}, and is updated every $10,000$ steps, as in \citet{van2016deep}. 

\paragraph{Network architecture} We use a single neural network to obtain estimates $\hat{\phi}$ and $\hat{\psi}$. 
\begin{enumerate}[itemsep=0pt,topsep=0em]
    \item Features: the neural network converts $4\times 84\times 84$ pixel states (obtained through standard frame max-pooling and stacking) into a $3136$-dimensional feature vector, using a convolution network with the same architecture as in \cite{mnih2015human}. 
    \item Hidden layer: the feature vector is then mapped to a hidden representation of size $1024$ by a fully connected layer followed by a ReLU activation.
    \item $\hat{\phi}$ prediction: the hidden representation is mapped to a size $64$ prediction of $\hat{\phi}$ for each action in $\actions$ by a fully connected layer with ReLU activation.
    \item $\hat{\psi}$ prediction: the hidden representation is mapped to $1 + |\actions|$ vectors of size $64$. The first vector gives the average successor features for that state $\bar{\psi}(s)$, whilst each of the $|\actions|$ vectors predicts an advantage $\tilde{\psi} (s, a)$. The overall successor feature prediction is given by $\hat{\psi}(s, a) = \bar{\psi}(s) + \tilde{\psi}(s, a)$.
    \item Linear $\hat{Q}^\pi$ and $\hat{r}$ prediction: a final linear layer with weights $\hat{w}$ maps $\hat{\phi}$ to reward prediction and $\hat{\psi}$ to Q value prediction with both predictors sharing weights.
\end{enumerate}

\paragraph{Hyperparameter selection} We used six games for hyperparameter selection: \textsc{Asterix}, \textsc{Enduro}, \textsc{Freeway}, \textsc{Hero}, \textsc{Qbert}, \textsc{Seaquest}, a subset of the games commonly used for this purpose \citep{munos2016safe}. 12 combinations of parameters in the `search set' column were tested (that is, not an exhaustive gridsearch), for a total of $12 \times 6 = 72$ full game runs, or approximately 33\% of the entire computational cost of the experiment. 

\begin{table}[h]
    \centering
    \caption{\label{tab:atari-parameters} Hyperparameters used for Successor Uncertainties in Atari 2600 experiments.}
    \begin{tabular}{lcc}
            \toprule
            Hyperparameter  & Search set & Value used \\
            \midrule
            Action repeat & --- &4 \\
            Train interval & --- &4 \\
            
            \rule{0pt}{14pt}Learning rate & $\{2.5\times 10^{-4}, 5\times 10^{-5}  \}$ &$5\times 10^{-5}$ \\
            Batch size & --- & 32 \\
            Gradient clip norm cutoff & --- &$10$ \\
            Target update interval & $\{10^3, 10^4 \}$ &$10^4$ \\
            
            \rule{0pt}{14pt}Successor feature size & $\{ 32, 64\}$ &64 \\
            Hidden layer size & --- & 1024 \\

            \rule{0pt}{14pt}Prior variance $\theta$ & --- & $1$ \\
            Likelihood variance $\beta$ & $\{10^{-3}, 10^{-2} \}$ &$10^{-3}$ \\
            $\hat{\Sigma}_w$ decay factor $\zeta$ & $ \{1- 10^{-5}, 1-10^{-4} \}$ &$1 - 10^{-5}$ \\        
            \bottomrule
    \end{tabular}
\end{table}

\end{appendices}
\end{document}